\documentclass{article}

\usepackage[preprint,nonatbib]{neurips_2021}

\usepackage[utf8]{inputenc} 
\usepackage[T1]{fontenc}    
\usepackage{hyperref}       
\usepackage{url}            
\usepackage{booktabs}       
\usepackage{amsfonts}       
\usepackage{nicefrac}       
\usepackage{microtype}      
\usepackage{amsmath}
\usepackage{subcaption}
\usepackage{algorithm}
\usepackage{algpseudocode}
\usepackage{wrapfig}
\usepackage{floatflt}

\usepackage{tabularx}
\usepackage{microtype}
\usepackage{graphicx}
\usepackage{booktabs} 
\usepackage{booktabs}
\usepackage[table,xcdraw]{xcolor}

\usepackage{comment}

\usepackage{color, soul}
\usepackage{tabularx}
\usepackage{lipsum}
\usepackage{graphicx}

\soulregister\cite7
\soulregister\ref7
\soulregister\eqref7
\soulregister\pageref7
\sethlcolor{yellow}

\usepackage{amsmath,amsfonts,amssymb}
\usepackage{amsthm}

\def \Nplus {\mathbb{N}_+}
\def \st {\ \ \textnormal{s.t.} \ \ }
\def \ST {\ \ \textnormal{subject to} \ \ }
\def \Re{\mathbb{R}}

\DeclareMathOperator*{\argmin}{arg\,min}
\DeclareMathOperator*{\argmax}{arg\,max}

\theoremstyle{plain}

\newtheorem{corollary}{Corollary}

\newtheorem{lemma}{Lemma}
\newtheorem{proposition}{Proposition}
\newtheorem{theorem}{Theorem}

\theoremstyle{remark}

\theoremstyle{definition}

\definecolor{orange-left}{rgb}{0.85, 0.325, 0.098}
\definecolor{blue-right}{rgb}{0, 0.447, 0.741}
\definecolor{alizarin}{rgb}{0.82, 0.1, 0.26}

\definecolor{A}{rgb}{0.6350, 0.0780, 0.1840}
\definecolor{B}{rgb}{0.5, 0, 0.5}
\definecolor{C}{rgb}{0.3010, 0.7450, 0.9330}
\definecolor{D}{rgb}{0.8500, 0.3250, 0.0980}
\definecolor{E}{rgb}{0.4660, 0.6740, 0.1880}
\definecolor{F}{rgb}{0.9290, 0.6940, 0.1250}
\definecolor{G}{rgb}{0, 0.4470, 0.7410}

\definecolor{R4}{rgb}{0.0, 0.4, 0.65}
\definecolor{R6}{rgb}{1.0, 0.22, 0.0}
\definecolor{R7}{rgb}{0.8, 0.0, 0.8}
\definecolor{R8}{rgb}{0.4, 0.22, 0.33}
\definecolor{R9}{rgb}{0.65, 0.04, 0.37}

\newcommand{\Ai}{{\color{A}{\textbf{A}}}}
\newcommand{\B}{{\color{B}{\textbf{B}}}}
\newcommand{\C}{{\color{C}{\textbf{C}}}}
\newcommand{\E}{{\color{E}{\textbf{E}}}}
\newcommand{\Dc}{{\color{D}{\textbf{D}}}}
\newcommand{\F}{{\color{F}{\textbf{F}}}}
\newcommand{\G}{{\color{G}{\textbf{G}}}}

\newcommand{\myparagraph}[1]{\smallskip\noindent\textbf{#1.}}

\def\I{\mathbf{I}}
\def\A{\mathbf{A}}
\def\d{\mathbf{d}}
\def\x{\mathbf{x}}
\def\Z{\mathbf{Z}}
\def\U{\mathbf{U}}
\def\V{\mathbf{V}}
\def\Re{\mathbb{R}}
\def\D{\mathbf{D}}
\def\X{\mathbf{X}}
\def\L{\mathbf{L}}
\def\Y{\mathbf{Y}}
\def\k{\mathbf{k}}

\makeatletter
\@addtoreset{theorem}{section}
\@addtoreset{proposition}{section}
\makeatother

\makeatletter
\def\SOUL@hlpreamble{%
   \setul{\dp\strutbox}{\dimexpr\ht\strutbox+\dp\strutbox\relax}%
   \let\SOUL@stcolor\SOUL@hlcolor
   \SOUL@stpreamble
}

\title{Wave-Informed Matrix Factorization with \\Global Optimality Guarantees}

\author{%
  Harsha Vardhan Tetali \\
  Department of Electrical and Computer Engineering\\
  University of Florida\\
  Gainesville, FL 32603 \\
  \texttt{vardhanh71@gmail.com} \\
  \And
  Joel B. Harley \\
  Department of Electrical and Computer Engineering\\
  University of Florida\\
  Gainesville, FL 32603 \\
  \texttt{joel.harley@ufl.edu} \\
  \And
  Benjamin D. Haeffele \\
  Mathematical Institute of Data Science\\
  Johns Hopkins University\\
  Baltimore, MD 21218\\
  \texttt{bhaeffele@jhu.edu}
}

\begin{document}

\maketitle

\begin{abstract}
 With the recent success of representation learning methods, which includes deep learning as a special case, there has been considerable interest in developing representation learning techniques that can incorporate known physical constraints into the learned representation. As one example, in many applications that involve a signal propagating through physical media (e.g., optics, acoustics, fluid dynamics, etc), it is known that the dynamics of the signal must satisfy constraints imposed by the wave equation. Here we propose a matrix factorization technique that decomposes such signals into a sum of components, where each component is regularized to ensure that it satisfies wave equation constraints. Although our proposed formulation is non-convex, we prove that our model can be efficiently solved to global optimality in polynomial time. We demonstrate the benefits of our work by applications in structural health monitoring, where prior work has attempted to solve this problem using sparse dictionary learning approaches that do not come with any theoretical guarantees regarding convergence to global optimality and employ heuristics to capture desired physical constraints. 
\end{abstract}


Representation learning has gained importance in fields that utilize data generated through physical processes, such as weather forecasting \cite{mehrkanoon2019deep}, manufacturing \cite{ringsquandl2017knowledge}, structural health monitoring \cite{zobeiry2020theory}, acoustics \cite{mohamed2012understanding}, and medical imaging \cite{litjens2017survey}. However, in general, the features learned through generic machine learning algorithms typically do not correspond to physically interpretable quantities. Yet, learning physically consistent and interpretable features can improve our understanding of 1) the physically viable information about the data and 2) the composition of the system or process generating it. As a result, in recent years modified learning paradigms, suited to different physical application domains, have begun to draw interest and discussion \cite{karpatne2017theory}. 
For example, several researchers have designed physics-informed neural networks to learn approximate solutions to a partial differential equation \cite{BenjaminErichson2019,Raissi2019,Nabian2018,Tartakovsky2018,Long2018}. These have been recently applied to ultrasonic surface waves to extract velocity parameters and denoise data \cite{Shukla2020}. Similar physics-guided neural networks \cite{karpatne2017theory,Karpatne2017a,Lei2018,Mestav2018} use physics-based regularizations to improve data processing, demonstrating that regression tasks (such as estimating the temperature throughout a lake) can be more accurate and robust when compared with purely physics-based solutions or purely data-driven solutions. These approaches demonstrate the strong potential for physics-informed machine learning but remain early in their study.

Likewise, several generative models of waves have recently been studied for applications in electroencephalography \cite{Luo2020} and seismology \cite{Mosser2020}, although these learning systems do not assume physical knowledge. Similarly, other prior work has also considered generative models of ultrasonic waves through the use of sparse signal processing \cite{harley2013sparse}, dictionary learning \cite{alguri2018baseline,alguri2017model,alguri2021sim}, and neural network sparse autoencoders \cite{Alguri2019}. These methods utilize the fact that many waves can be expressed as a small, sparse sum of spatial modes and attempt to extract these modes within the data \cite{alguri2018baseline}. These dictionary learning algorithms have further been combined with wave-informed regularizers to incorporate physical knowledge of wave propagation into the solution \cite{Tetali2019}.

However, despite the wide-ranging applications for representation learning models for physics-constrained problems, one of the key challenges associated with representation learning tasks is that they typically require one to solve a non-convex optimization problem. While this potentially presents considerable challenges due to the general difficulty of non-convex optimization, considerable attention has been devoted recently to the non-convex optimization problems that arise from representation learning problems, and despite non-convex optimization being hard in the general case, certain positive results have begun to emerge in certain settings.  For example, within the context of low-rank matrix recovery, it has been shown that, in certain circumstances, gradient descent is guaranteed to converge to a local minimum and all local minima will be globally optimal \cite{bhojanapalli2016global, ge2016matrix, park2016non, ge2017no}.
Further, other recent work has studied `structured' matrix factorization problems, where one promotes various properties in the matrix factors by imposing some form of regularization on the factorized matrices to promote the desired properties (e.g., an $\ell_1$ norm on a matrix factor when sparsity is desired in that factor) and shown that such problems can be solved to global optimality in certain circumstances, but whether such guarantees can be made depends critically on how the regularization on the factorized matrices is formulated in the model \cite{haeffele2014structured,haeffele2019structured,bach2013convex}.

\myparagraph{Paper Contributions and Related Work} In this work, we will specifically focus on data that is obtained from wave-like phenomena and note that many sensing and communications applications for representation learning center around the interpretation of waves and the wave equation to interrogate and characterize systems.  
For example, there are numerous applications for representation learning to aid in understanding and designing engineering systems with both electromagnetic waves (e.g., antenna design \cite{Testolina_2019} and communication systems \cite{7528397}) and mechanical waves (e.g., structural health monitoring \cite{alguri2018baseline} and micro-electromechanical device design \cite{7936037}).  Specifically, we derive a novel matrix factorization model which ensures that the recovered components satisfy wave equation constraints, allowing for direct interpretability of the recovered factors based on known physics and decomposing signals which consist of multiple oscillatory modes into the various components. Although our resulting formulation will also require solving a non-convex optimization problem, we show, using prior work on structured matrix factorization \cite{haeffele2014structured,haeffele2019structured, bach2013convex}, that our model can be solved to global optimality in polynomial time.  We demonstrate our approach's advantages for two example applications, electrical waves in a multi-segment transmission line and mechanical waves on a fixed string. The first application is inspired from \textbf{material characterization}, where different materials have different wavenumbers, as arises in the health monitoring of transmission lines and structures \cite{ernould2020characterization,Shukla2020, shukla2021physics}. The second application provides insights into how enforcing PDE constraints is beneficial in \textbf{modal analysis}, such as in the study of cantilever beams \cite{lai2020full}.

Prior work along these lines was given in \cite{Tetali2019}, but in comparison, our algorithm provides global optimality guarantees (along with a certificate of guaranteed optimality), which is difficult to achieve for such non-convex problems. In addition, the formulation in \cite{Tetali2019} is incomplete, as it does not account for separability in deriving the formulation, which we motivate from separation of variables method. This work also differs in having a low-rank assumption on the matrix factorization which enables us to extend our method to matrix completion strategies on incomplete data.  Further, \cite{Tetali2019} considers a dictionary learning approach of the data matrix after taking a Fourier transform to obtain a sparse frequency representation, whereas our approach can directly be applied in the time domain.

\section{A Model for Wave-Informed Matrix Factorization}
\label{others}

In this section, we develop a novel matrix factorization framework which will allow us to decompose a matrix $\Y$, with each column representing sampling along one dimension (e.g., space), as the product of two matrices $\D \X^\top$, where one of the matrices will be constrained to satisfy wave-equation constraints.  Additionally, our model will learn the number of modes (columns in $\D$ and $\X$) directly from the data without needing to specify this \textit{a priori}.  

\myparagraph{Problem Formulation}
To begin, recall that a key concept in the solution of partial differential equations is the notion of separation.  For example, 
given a PDE in space $\ell$ and time $t$, we assume the solution is of the form $f(\ell,t) = d(\ell) x(t)$ and try to solve it. In case of linear PDEs (like the wave equation), for some $N$, if we find that $d_i(\ell) x_i(t)$ are solutions for all $i \in [N]$, then we have that $\sum_{i=1}^{N} d_i(\ell) x_i(t)$ is also a solution. For a discrete approximation of the solution $\sum_{i=1}^{N} d_i(\ell) x_i(t)$, we first note that the product of continuous functions $d_i(\ell) x_i(t)$ on $\ell \in [0,L]$ and $t \in [0,T]$ can be approximated as an outer product of two vectors, leading directly to a matrix factorization model. We show this as follows: consider matrices $(\D,\X)$ whose $i^\text{th}$ columns are defined as $\mathbf{D}_i = \left[ d_i(0), d_i(\Delta \ell), \cdots, d_i(N_d \Delta \ell)  \right]^{\top}$ and $\mathbf{X}_i = \left[ x_i(0), x_i(\Delta t), \cdots, x_i(N_t \Delta t)  \right]^{\top}$ (where $N_d = \left \lfloor{L/\Delta \ell}\right \rfloor$ and $N_t = \left \lfloor{T/\Delta t}\right \rfloor$), giving $d_i(n \Delta \ell) x_i(m \Delta t) = \mathbf{D}_{n,i} \mathbf{X}_{m,i}$. 

Thus we have that a discrete approximation of $d_i(\ell) x_i(t)$ is $\mathbf{D}_i \mathbf{X}_i^{\top}$ 
and likewise, the discrete approximation of $\sum_{i=1}^{N} d_i(\ell) x_i(t)$ is $ \sum_{i=1}^{N} \mathbf{D}_i \mathbf{X}_i^{\top} = \mathbf{D} \mathbf{X}^{\top}$. 

From the above discussion, the fact that we are trying to approximate the wave-data matrix $\mathbf{Y}$ as a product $\mathbf{D} \mathbf{X}^{\top}$ clearly results in a matrix factorization model. However, beyond simple matrix factorization, we also wish to enforce physical consistency in our model. For this we employ theory-guided regularization (see \cite{karpatne2017theory}) and the regularizer is derived from the wave equation. Specifically, we need a cost function which 1) captures how well our matrix factorization model matches the observed samples as well as 2) a regularization term which enforces physical consistency in the matrix factors, which we formulated as a problem with the form:
\begin{equation}
    \underset{\mathbf{D}, \mathbf{X}}{\text{ } \min \text{ }} { \tfrac{1}{2}\| \mathbf{Y} - \mathbf{D}\mathbf{X}^{\top} \|^2_F + \lambda \Theta (\mathbf{D}, \mathbf{X})}
    \label{eqn:obj_final}
\end{equation}
where $\| \mathbf{Y} - \mathbf{D}\mathbf{X}^{\top} \|^2_F$ considers the matrix factorization loss, $\lambda > 0$ is a hyper-parameter to balance the trade-off between fitting the data and satisfying our model assumptions given by $ \Theta (\mathbf{D}, \mathbf{X})$, which is a regularizer that promotes the desired physical consistency. We call this \emph{wave-informed matrix factorization}. To derive the specific function $\Theta (\mathbf{D}, \mathbf{X})$ which promotes the desired physical consistency, we begin by designing a regularizer to promote wave-consistent signals in the columns of $\D$.\footnote{Note that our approach is also generalizable to having wave-consistent regularization in both factors.} 
Specifically, consider the 1-dimensional wave equation evaluated at $d_i(l)x_i(t)$:
\begin{eqnarray}
    \frac{\partial^2 \left[ d_i(l)x_i(t) \right] }{\partial l^2} &=& \frac{1}{c^2} \frac{\partial^2 \left[ d_i(l)x_i(t) \right] }{\partial t^2} 
    \label{eqn:wave_equation}
\end{eqnarray}
Note that the above equation constrains that each 
each $d_i(l)x_i(t)$ component ($i \in [N]$) is a wave. This can thus be seen as decomposing the given wave data into a superposition of simpler waves. The Fourier transform in time on both sides of \eqref{eqn:wave_equation} yields:
$\frac{\partial^2 \left[ d_i(l) \right] }{\partial l^2}X_i(\omega) = \frac{-\omega^2}{c^2} d_i(l) X_i(\omega)$. Thus, at points where $X_i(\omega) \neq 0$, we can enforce the equation,
\begin{eqnarray}
    \frac{\partial^2 \left[ d_i(l) \right] }{\partial l^2} &=& \frac{-\omega^2}{c^2} d_i(l) 
    \label{eqn:space_eigen_value}
\end{eqnarray}
to maintain consistency with the wave-equation. 
Note that the above equation is analytically solvable for constant $\omega/c$ and the solution takes the form $d_i(l) = A \sin(\omega l/c) + B \cos(\omega l/c)$ and the constants $A$ and $B$ can be determined by initial conditions. Thus, we conclude that the choice of $\omega/c$ determines $d_i(l)$ and vice versa. 
We next discretize equation \eqref{eqn:space_eigen_value}. From the literature on discretizing differential equations, the above can be discretized to:
\begin{eqnarray}
    \mathbf{L} \mathbf{D}_i &=& -k_i^2 \mathbf{D}_i 
    \label{eqn:discrete_wave_eqn}
\end{eqnarray}
where $k_i = \omega / c$ , also known as the wavenumber and $\L$ is a discrete second derivative operator, where the specific form of $\L$ depends on the boundary conditions (e.g., Dirichlet, Dirichlet-Neumann, etc) which can be readily found in the numerical methods literature (see for example \cite{strang2007computational, golub2013matrix}).

We also specify the dependence on $i$ owing to the fact that $\omega/c$ and $d_i(l)$ determine each other.  
Now for the factorization to be physically consistent in terms of spatial data, we desire that \eqref{eqn:discrete_wave_eqn} is satisfied for some (unknown) value of $k_i$, which we promote with the regularizer $\min_{k_i} \| \mathbf{L} \mathbf{D}_i + k_i^2 \mathbf{D}_i \|^2_F$. In addition to requiring that the columns of $\D$ satisfy the wave equation, we also would like to constrain the number of modes in our factorization (or equivalently find a factorization $\D \X^\top$ with rank equal to the number of modes).  To accomplish this, we add squared Frobenius norms to both $\D$ and $\X$, which is known to induce low-rank solutions in the product $\D \X^\top$ due to connections with the variational form of the nuclear norm \cite{haeffele2019structured,haeffele2014structured,srebro2005rank}.  Taken together, we arrive at a regularization function which decouples along the columns of $(\D,\X)$ given by:
\begin{equation}
\label{eq:theta}
\Theta(\D,\X) = \sum_{i=1}^N \theta(\D_i,\X_i), \text{ where }  
   \theta(\D_i,\X_i) = \tfrac{1}{2} \! \! \left( \|\X_i\|_F^2 \! + \! \|\D_i\|_F^2 \right) \!  +  \! \gamma \min_{k_i} \|\L \D_i \! + \! k_i^2 \D_i \|_F^2
\end{equation}
In the supplementary material we show that the above regularization on $\D$ is equivalent to encouraging the columns of $\D$ to lie in the passband of a bandpass filter (centered at wavenumber $k$) with the choice of the hyperparameter $\gamma$ being inversely proportional to the filter bandwidth.  With this regularization function, we then have our final model that we wish to solve:
\begin{equation}
\label{eq:main_obj}
\min_{N \in \Nplus} \min_{\substack{\D \in \Re^{L \times N}, \X \in \Re^{T \times N} \\ \k \in \Re^N}} \tfrac{1}{2}\|\Y-\D\X^\top\|_F^2 + \tfrac{\lambda}{2} \sum_{i=1}^N \left(\|\X_i\|_F^2 + \|\D_i\|_F^2 + \gamma \|\L \D_i + k_i^2 \D_i \|_F^2 \right).
\end{equation}

\section{Model Optimization with Global Optimality Guarantees}
\label{section:model}
We note that our model in \eqref{eq:main_obj} is inherently non-convex in $(\D,\X,\k)$ jointly due to the matrix factorization model as well as the fact that we are additionally searching over the number of columns/entries in $(\D,\X,\k)$, $N$.  However, despite the potential challenge of non-convex optimization, here we show that we can efficiently solve \eqref{eq:main_obj} to global optimality in polynomial by leveraging prior results regarding optimization problems for structured matrix factorization \cite{haeffele2019structured,bach2013convex}.  In particular, the authors of \cite{haeffele2019structured} consider a general matrix factorization problem of the form:
\begin{equation}
\label{eq:gen_obj}
\min_{N\in\Nplus}{\min_{\U \in \Re^{m \times N}, \V \in \Re^{n \times N}}} \ell(\U \V^\top) + \lambda \sum_{i=1}^N \bar \theta(\U_i,\V_i)
\end{equation}
where $\ell(\hat \Y)$ is any function which is convex and once differentiable in $\hat \Y$ and $\bar \theta(\U_i,\V_i)$ is any function which satisfies the following 3 conditions:
\begin{enumerate}
\item $\bar \theta(\alpha \U_i, \alpha \V_i) = \alpha^2 \theta(\U_i, \V_i), \ \forall (\U_i,\V_i)$, $\forall \alpha \geq 0$.
\item $\bar \theta(\U_i, \V_i) \geq 0, \ \forall (\U_i,\V_i)$.
\item For all sequences $(\U_i^{(n)},\V_i^{(n)})$ such that $\|\U_i^{(n)}(\V_i^{(n)})^\top\| \rightarrow \infty$ then $\bar \theta(\U_i^{(n)},\V_i^{(n)}) \rightarrow \infty$.
\end{enumerate}
Clearly, our choice of loss function $\ell$ (squared loss) satisfies the necessary conditions.  However, in \eqref{eq:main_obj} we wish to optimize over not just the matrix factors $(\D,\X)$ but also the additional $\k$ parameters, so it is not immediately apparent that the framework from \cite{haeffele2019structured} can be applied to our problem.  Here we first show that by our design of the regularization function $\theta$ our formulation will satisfy the needed conditions, allowing us to apply the results from \cite{haeffele2019structured} to our problem of interest \eqref{eq:main_obj}.
\begin{proposition}
\label{prop:equiv_prob}
The optimization problem in \eqref{eq:main_obj} is a special case of the problem considered in \cite{haeffele2019structured}.
\end{proposition}
\begin{proof}
All proofs can be found in the supplement.
\end{proof}
From this, we note that within the framework developed in \cite{haeffele2019structured} it is shown (Corollary 1) that a given point $(\tilde \U, \tilde \V)$ is a globally optimal solution of \eqref{eq:gen_obj} iff the following two conditions are satisfied:
\begin{equation}
1) \ \  \langle -\nabla \ell( \tilde \U \tilde \V^\top), \tilde \U \tilde \V^\top \rangle = \lambda \sum_{i=1}^N \bar \theta(\tilde \U_i, \tilde \V_i) \ \ \ \ \ \ 2) \ \ \Omega_{\bar \theta}^\circ (-\tfrac{1}{\lambda} \nabla \ell(\tilde \U \tilde \V^\top)) \leq 1
\end{equation}
where $\nabla(\ell(\hat \Y))$ denotes the gradient w.r.t. the matrix product $\hat \Y = \U \V^\top$ and $\Omega_{\bar \theta}^\circ(\cdot)$ is referred to as the `polar problem' which is defined as 
\begin{equation}
\label{eq:polar_def}
\Omega_{\bar \theta}^\circ (\Z) \equiv \sup_{\mathbf{u},\mathbf{v}} \mathbf{u}^\top \Z \mathbf{v} \st \bar \theta(\mathbf{u},\mathbf{v}) \leq 1.
\end{equation}
It is further shown in \cite{haeffele2019structured} (Proposition 3)  that the first condition above will always be satisfied for any first-order stationary point $(\tilde \U, \tilde \V)$, and that if a given point is not globally optimal then the objective function \eqref{eq:gen_obj} can always be decreased by augmenting the current factorization by a solution to the polar problem as a new column:
\begin{equation}
(\U, \V) \! \! \leftarrow \! \! \left( \left[ \tilde \U, \ \tau \mathbf{u}^* \right], \left[ \tilde \V, \ \tau \mathbf{v}^* \right] \right)  \! : \!  \mathbf{u}^*, \mathbf{v}^* \in \argmax_{\mathbf{u},\mathbf{v}} \mathbf{u}^\top \! (-\tfrac{1}{\lambda} \nabla \ell(\tilde \U \tilde \V^\top) \mathbf{v} \st \bar \theta(\mathbf{u},\mathbf{v}) \leq 1 
\end{equation}
for an appropriate choice of step size $\tau > 0$.
If one can efficiently solve the polar problem for a given regularization function $\bar \theta$, then this provides a means to efficiently solve problems with form \eqref{eq:main_obj} to global optimality with guarantees of global optimality.  Unfortunately, the main challenge in applying this result from a computational standpoint is that solving the polar problem requires one to solve another challenging non-convex problem in \eqref{eq:polar_def}, which is often NP-Hard for even relatively simple regularization functions \cite{bach2013convex}. Here, however, we provide a key positive result, proving that for our designed regularization function \eqref{eq:theta}, the polar problem can be solved efficiently, in turn enabling efficient and guaranteed optimization of the non-convex model \eqref{eq:main_obj}.
\begin{theorem}
\label{thm:polar}
For the objective in \eqref{eq:main_obj}, the associated polar problem is equivalent to:
\begin{align}
\Omega_\theta^\circ (\Z) &= \max_{\d \in \Re^{L}, \x \in \Re^{T}, k\in \Re} \d^\top \Z \x
\mathrm{s.t.} \ &\|\d\|^2_F + \gamma \|\L\d + k^2 \d\|_F^2 \leq 1, 
 &\|\x\|^2_F \leq 1, \ 0 \leq k \leq 2.
\end{align}
Further, let $\L = \Gamma \Lambda \Gamma^\top$ be an eigen-decomposition of $\L$ and define the matrix $\A( \bar k) = \Gamma(\I + \gamma(\bar k \I + \Lambda)^2) \Gamma^\top$.  Then, if we define $\bar k^*$ as%
\begin{equation}
\bar k^* = \argmax_{\bar k \in [0,4]} \| \A(\bar k)^{-1/2} \Z \|_2
\label{eq:k_max}
\end{equation}
optimal values of $\d,\x, k$ are given as $\d^* = \A(\bar k^*)^{-1/2} \bar \d$, $\x^* = \bar \x$, and $k^* = (\bar k^*)^{1/2}$ where $(\bar \d, \bar \x)$ are the left and right singular vectors, respectively, associated with the largest singular value of $\A(\bar k^*)^{-1/2} \Z$.  Additionally, the above line search over $\bar k$ is Lipschitz continuous with a Lipschitz constant, $L_{\bar k}$, which is bounded by:
\begin{equation}
L_{\bar k} \leq \begin{Bmatrix} \frac{2}{3 \sqrt{3}} \sqrt{\gamma} \|\Z\|_2 & \gamma \geq \frac{1}{32} \\ 
4 \gamma (1 + 16 \gamma)^{-\tfrac{3}{2}} \|\Z\|_2 & \gamma < \frac{1}{32} \end{Bmatrix} \leq \tfrac{2}{3 \sqrt{3}} \sqrt{\gamma} \|\Z\|_2
\end{equation}
\end{theorem}

We also note that the above result implies that we can solve the polar by first performing a (one-dimensional) line search over $k$, and due to the fact that the largest singular value of a matrix is a Lipschitz continuous function, this line search can be solved efficiently by a variety of global optimization algorithms.  For example, we give the following corollary for the simple algorithm given in \cite{malherbe2017global}, and similar results are easily obtained for other algorithms.

\begin{corollary}[Adapted from Cor 13 in \cite{malherbe2017global}]
\label{cor:line_search}
For the function $f(\bar k) = \|\A(\bar k)^{-1/2} \Z \|_2$ as defined in Theorem \ref{thm:polar}, if we let $\bar k_1, \ldots \bar k_r$ denote the iterates of the LIPO algorithm in \cite{malherbe2017global} then we have $\forall \delta \in (0,1)$ with probability at least $1-\delta$,
\begin{equation}
\max_{\bar k \in [0,4]} f(\bar k) - \max_{i=1\ldots r} f(\bar k_i) \leq \tfrac{8}{3 \sqrt{3}} \sqrt{\gamma} \|\Z\|_2 \frac{\ln(1/\delta)}{r}
\end{equation}
\end{corollary}
As a result, we have that the error of the linesearch converges with rate $\mathcal{O}(1/r)$, where $r$ is the number of function evaluations of $f(\bar k)$, to a global optimum, and then, given the optimal value of $k$, the optimal $(\d, \x)$ vectors can be computed in closed-form via singular value decomposition.  Taken together this allows us to employ the Meta-Algorithm defined in Algorithm \ref{alg:meta} to solve problem \eqref{eq:main_obj}.
\begin{corollary}
\label{cor:poly_time}
Algorithm \ref{alg:meta} produces an optimal solution to \eqref{eq:main_obj} in polynomial time.
\end{corollary}

\begin{algorithm}
\caption{\bf{Meta-algorithm}}
\label{alg:meta}
\begin{algorithmic}[1]
\State Input $\mathbf{D}_{init}$, $\mathbf{X}_{init}$, $\mathbf{k}_{init}$ 
\State Initialize $ \left( \mathbf{D}, \mathbf{X}, \mathbf{k} \right) \leftarrow \left( \mathbf{D}_{init}, \mathbf{X}_{init}, \mathbf{k}_{init} \right)$
\While {global convergence criteria is not met}
\State Perform gradient descent on the low-rank wave-informed objective function \eqref{eq:main_obj} with $N$ fixed to reach a first order stationary point $(\Tilde{\mathbf{D}}, \Tilde{\mathbf{X}}, \Tilde{\mathbf{k}})$\label{step:grad_desc}
\State Calculate the value of $\Omega^\circ_\theta(\tfrac{1}{\lambda}(\Y-\tilde \D \tilde \X^\top))$ via Theorem \ref{thm:polar} above and obtain $\mathbf{d}^*, \mathbf{x}^*, k^*$
\If {value of polar $\Omega^\circ_\theta(\tfrac{1}{\lambda}(\Y-\tilde \D \tilde \X^\top)) = 1$} \State {Algorithm converged to global minimum}
\Else 
\State {Append  $(\mathbf{d}^*, \mathbf{x}^*, k^*)$ to $(\Tilde{\mathbf{D}}, \Tilde{\mathbf{X}}, \Tilde{\mathbf{k}})$ and update $\left( \mathbf{D}, \mathbf{X}, \mathbf{k} \right)$
\State $ \left( \mathbf{D}, \mathbf{X}, \mathbf{k} \right) \! \leftarrow \! \left( \left[ \tilde{\mathbf{D}}, \ \tau \mathbf{d}^*  \right], \left[ \tilde{\mathbf{X}}, \ \tau \mathbf{x}^*  \right] , \left[ \tilde{\mathbf{k}}^\top, \ k^*  \right]^\top
\right)$, $\tau\! > \! 0$ is step size (see supplement).\label{step:step_size}}
\EndIf
\State Continue loop.
\EndWhile
\end{algorithmic}
\label{algoblock:meta-algo}
\end{algorithm}

\vspace{-3mm}

\section{Results \& Discussions}
\label{sec:results}
In this section, we evaluate our proposed wave-informed matrix factorization on two synthetic datasets. For each dataset, algorithm iterations are performed until the value of polar is evaluated to be less than a threshold close to 1 to demonstrate convergence to the global minimum (i.e., a polar value of 1).  We note that the distance to the global optimum (in objective value) is directly proportional to the value of the polar value at any point minus 1 (\cite{haeffele2019structured}, Prop. 4), so choosing a stopping criteria as polar value $\leq 1 \! + \! \epsilon$ also guarantees optimality within $\mathcal{O}(\epsilon)$.

\begin{figure}
    \centering
    \includegraphics[scale=0.25]{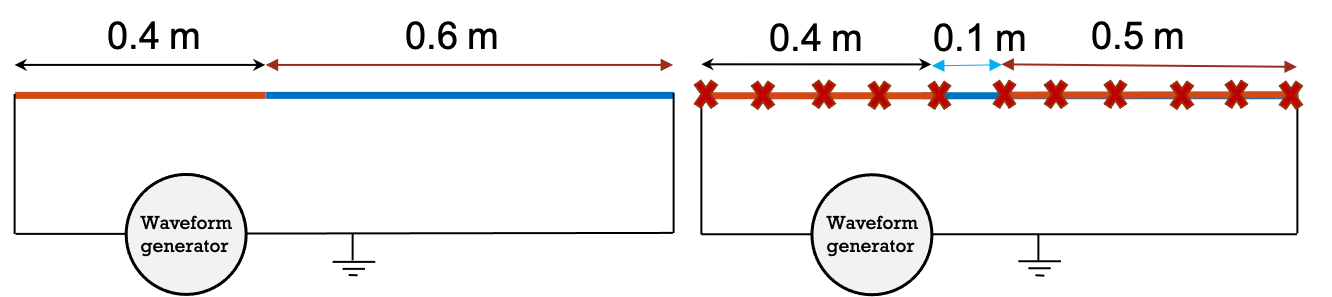}
    \caption{(left) Two cables of different impedances joined end-to-end; (right) cables of two different materials with a slightly alternate configuration, and measurements sampled at \color{alizarin} \textbf{X} \color{black}.}
    \label{fig:merge_1}
    \vspace{-5mm}
\end{figure}

\myparagraph{Characterizing Multi-Segment Transmission Lines} First, we consider the \textbf{material characterization} problem described in the introduction.
Specifically, we consider waves propagating along an electrical transmission line where the impedance of the transmission line changes along the line (as would occur, for example, if the transmission line was damaged or degraded in a region), as depicted in Figure \ref{fig:merge_1}(left). 
We simulate wave propagation in this line based on a standard RLGC (resistance, inductance, conductance, capacitance) transmission line model \cite{Pozar2012}, using an algebraic graph theory engine for one-dimensional waves \cite{Harley2019}. For this simulation, we combine two transmission line segments of length $0.4$~m (\color{orange-left}left\color{black}) and $0.6$~m (\color{blue-right}right\color{black}). 

The wavenumber (which is inversely proportional to velocity) of the right segment is $4$~times higher than the wavenumber of the left segment. A modulated Gaussian signal with a center frequency of $10$~MHz and a $3$~dB bandwidth of $10$~MHz is transmitted from the left end of this setup. This impulse travels $0.4$~m till it reaches the interface (i.e., the connection between the two cables). At the interface, part of the propagating signal reflects and another part transmits into the next cable. The signal again reflects at the end of each cable.  

Note the following: 1) there are now two distinct wavenumber regions, 2) the excitation/forcing function is transient and is therefore a linear combination of nearby frequencies that each have a unique wavenumber, and 3) the end boundaries have loss (i.e., energy exits the system). We show electrical voltage amplitude at a timestamp, see Figure \ref{fig:Merge_2}(top-left). We observe waves travelling in two regions with different wavenumbers. Note that performing a simple Fourier transform (in space) of the signal will produce unsatisfactory results due to the discrete change in the wavenumber along the line.  However, as we show below, our method automatically recovers the two distinct regions as separated modes in our decomposition. Namely, we solve our wave-informed matrix factorization with $\gamma=50$ and $\lambda=0.6$ on the above described wave data. 
Inspection of the columns of $\D$, see Figure \ref{fig:Merge_2}(top-right), clearly shows that single columns of $\D$ contain energy largely contained to only one region of the transmission line, automatically providing a natural decomposition of the signal which corresponds to changes in the physical transmission line. 
\begin{figure}
    \centering
    \includegraphics[scale=0.093]{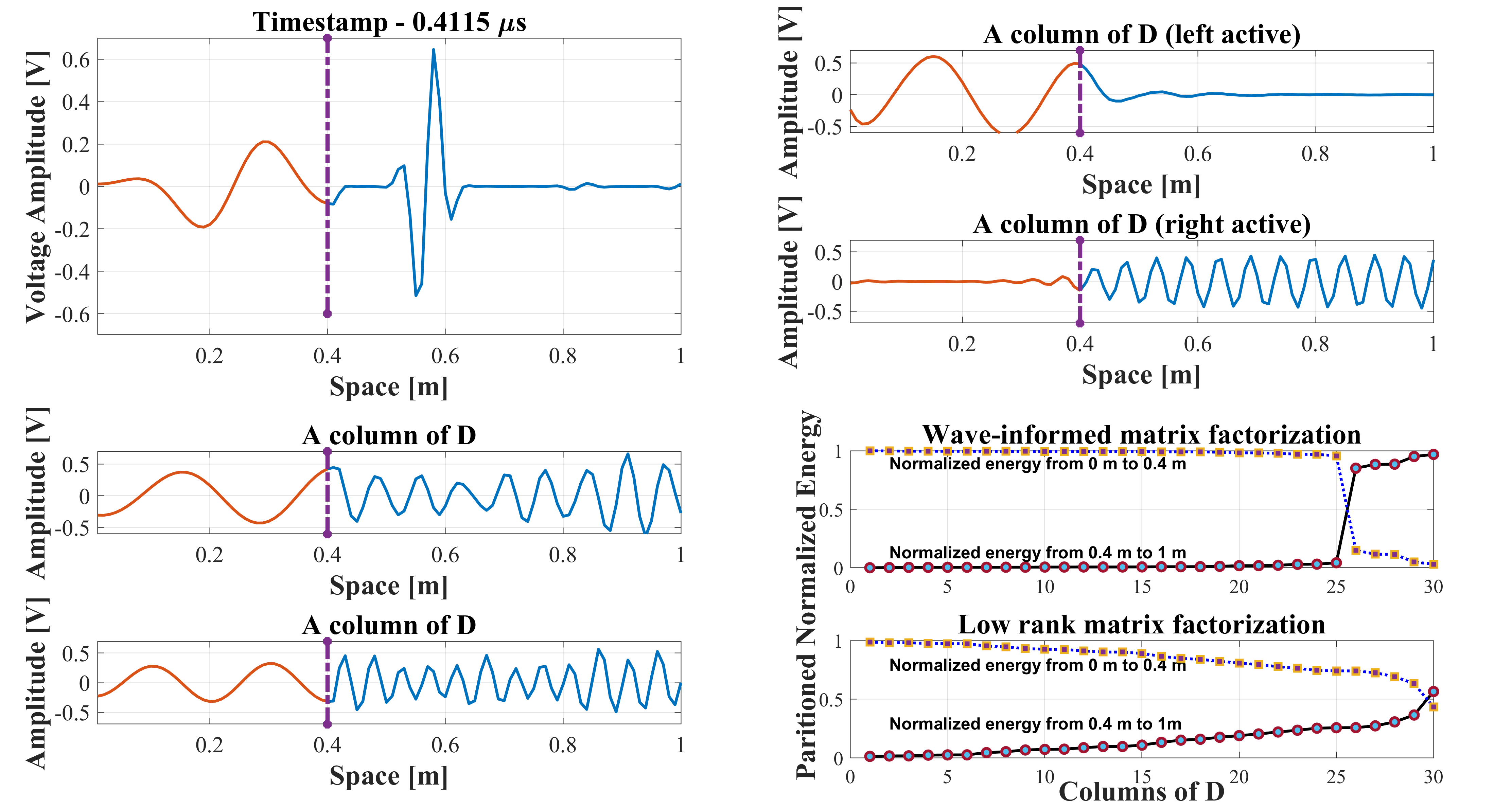}
    \caption{(top-left) Electrical amplitude at a timestamp of $0.4115 \mu$s; (top-right) illustration of two columns of $\D$ from Wave-informed matrix factorization; (bottom-left) illustration of two columns of $\D$ from low-rank matrix factorization;(bottom-right) partitioned normalized energy of the first 30 significant columns of Wave-informed matrix factorization and low-rank matrix factorization.}
    \label{fig:Merge_2}
    \vspace{-5mm}
\end{figure}

As a first baseline comparison, we also perform a low-rank factorization (by setting $\gamma=0$ and $\lambda=0.6$), which is equivalent to only using nuclear norm regularization on the matrix product $(\D \X^\top)$. 
With just low-rank regularization, one observes that while there is a clear demarcation at $0.4$~m, see Figure \ref{fig:Merge_2}(bottom-left), there is still significant energy in both transmission line segments for each column of $\D$. 
To quantitatively evaluate the quality of the decomposition, we normalize the 30 most significant (determined by the corresponding value of $\|\D_i \X_i^{\top} \|_F^2$) columns of $\D$ and plot the energy on each partition, see Figure \ref{fig:Merge_2}(bottom-right). In case of the low-rank factorization, see Figure \ref{fig:Merge_2}(bottom-right), 
we largely observe higher energy levels on regions of higher wavenumber (energy is proportional to the wavenumber for oscillatory quantities) and lower energy levels on regions of lower wavenumber. 
In the case of our wave-informed matrix factorization model, we see two clear step functions, one indicating the energy on the first segment and the other indicating energy on the second segment, showing a sharp decomposition of the signal into components corresponding to the two distinct regions. We can further encapsulate this into a single quantity representing the decomposition of the signal into components corresponding to the two distinct regions. For this we compute the mean entropy (mean over the columns of $\D$) of the percentage of signal power in one of the two regions (note the entropy is invariant to the choice of region). In the ideal case, the entropy will be 0, indicating that the learned components have all of the power being exclusively in one of the two line segments, while an entropy of 1 indicates the worst case where the power in the signal of a component column is split equally between the two segments.
 To quantify the performance of our method we also compare against other similar matrix factorization or modal extraction techniques: wave-informed KSVD \cite{Tetali2019} (\B), low-rank matrix factorization (\C), independent component analysis \cite{hyvarinen2000independent,van2006pp} (\Dc), dynamic mode decomposition \cite{PhysRevFluids.5.054401,SUSUKI2018327} (\E), ensemble empirical mode decomposition \cite{wu2009ensemble, fang2011stress, 6707404} (\F) and principal component analysis \cite{medina1993three} (\G). Dynamic mode decomposition (\E), independent component analysis (\Dc), and principal component analysis (\G) are subspace identification techniques that resemble matrix factorization in many respects. Ensemble empirical mode decomposition is a refined version of empirical mode decomposition, extended to 2D data), which learns a basis from data, similar to our algorithm.
 Note that our proposed wave-informed matrix factorization (\Ai) achieves the best performance of the seven methods, providing quantitative evidence of our model's ability to isolate distinct materials (and corresponding wavefields). 

\begin{table}[ht]
\vspace{-5mm}
\caption{\label{table:entropies} Comparison of entropy based performance for segmentation of a transmission line using algorithms (\Ai), (\B), (\C), (\Dc), (\E), (\F) and (\G), where the mean entropy row represents the mean entropy of partitioned normalized energies for each algorithm.}
\centering
\begin{tabularx}{\textwidth}{@{}lXXXXXXX@{}}
\toprule
\textbf{Algorithm}    & \Ai    & \B     & \C     & \Dc    & \E    & \F     & \G     \\ \midrule
\textbf{Mean Entropy} & 0.176 & 0.211 & 0.531 & 0.464 & 0.431 & 0.636 & 0.531 \\ \bottomrule
\end{tabularx}
\vspace{-5mm}
\end{table}

\myparagraph{Characterizing Multi-segment Transmission Lines with Sparsely Sampled Data}
In this subsection, we demonstrate the effectiveness of our model when the data is sparsely sampled. Specifically, we reduce the spatial density at which points were sampled (keeping the number of time points the same) by sampling 10 points uniformly on the cable of length $1$m, and we also change the configuration of the transmission lines by assigning the \color{orange-left} first $0.4$ m \color{black} and the \color{orange-left} last $0.5$ m \color{black} with one material and the \color{blue-right} middle $0.1$ m \color{black} with another material (see Fig.~\ref{fig:merge_1}(right)). We then solve a matrix completion problem with wave-informed matrix factorization to interpolate the wavefield at the remaining 90 points on the $1$m region.  Since the sampling is uniformly done over space, the data matrix contains all zero columns which implies it \textit{cannot} be completed with standard low-rank matrix completion.  However, we show that wave-informed matrix factorization, on the other hand, fills in those regions appropriately due to the wave-constraint being enforced.

Specifically, we minimize the following objective\footnote{The mathematical details of the optimization procedure are almost the same as the algorithm mentioned in Algorithm \ref{alg:meta} except that the masking operator $\mathcal{A}(\cdot)$ also needs to be included (see supplement).}:
\begin{equation}
\label{eq:main_obj_missing}
\!\!\min_{N \in \Nplus} \!\!\min_{\substack{\D \in \Re^{L \times N}, \\ \X \in \Re^{T \times N}, \k \in \Re^N}} \!\!\!\!\tfrac{1}{2}\|  \mathcal{A} \left( \Y-\D\X^\top \right) \|_F^2  + \tfrac{\lambda}{2} \sum_{i=1}^N \left(\|\X_i\|_F^2 + \|\D_i\|_F^2 + \gamma \|\L \D_i + k_i^2 \D_i \|_F^2 \right) 
\end{equation}
which is a matrix completion problem with a linear masking operator $\mathcal{A}(\cdot)$.  We observe, despite the very sparse sampling, the columns of $\D$ still maintain sufficient structure to identify different material regions. 
In Figure~\ref{fig:Merge_3}(left), we plot all of the recovered columns of $\D$, where we observe that, except for one column of $\D$, every other column contains very little energy in the region between $0.45$m and $0.55$m. Thus the algorithm again automatically detects a change in this region in the decomposition. Note that the actual change was introduced between $0.4$m and $0.5$m. An error of $0.05$m in estimating the region of material change is analogous to Gibbs phenomenon in signals and system theory -- this especially occurs due to the fact that we impose second derivative constraints on the columns of $\D$, which imposes a smoothness constraint and shifts the transition. Figure \ref{fig:Merge_3}(right) is similar to Figure \ref{fig:Merge_2}(bottom-left) as it quantifies that only one column of $\D$ is active in the middle region, and the other columns of $\D$ are active in the other distinct regions, whereas a low-rank factorization model displays significantly more mixing of signal energies in each region. 

\begin{figure}
    \centering
    \includegraphics[scale=0.09]{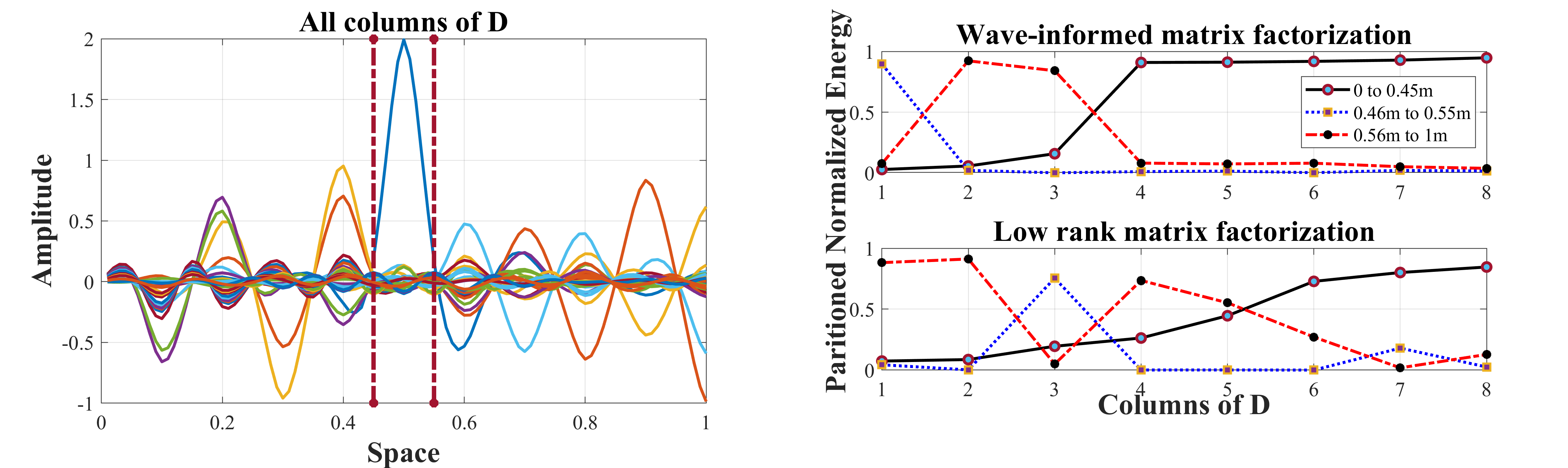}
    \caption{(left) All the columns of $\D$ recovered from wave-informed matrix factorization in a single graph, note that the measurements were sampled only at $ \{ 0,0.1, \cdots, 0.9 \}$; (right) normalized energies of various regions, for the first $8$ significant columns of $\D$ for both algorithms.}
    \label{fig:Merge_3}
    \vspace{-5mm}
\end{figure}

\myparagraph{A Fixed Vibrating String}
A fixed vibrating string is an example of a wave with fixed boundary conditions. 
As an example, consider the dynamics of a noisy fixed string, given by:  
\begin{equation}
y(\ell,t) = \sum_{n=1}^{R} a_n \sin (n k \ell) \sin(n \omega t) + \eta(\ell,t) \; .
\label{eqn:wave_equ}
\end{equation}
This can be represented in matrix notation as $\Y = \D \X^{\top} + \mathbf{N}$, where $\mathbf{N}$ is an additive noise term.  To demonstrate our framework, we will consider two more challenging variants of \eqref{eqn:wave_equ} where we add a damping factor (in either time or space) to the amplitude of the wave:
\begin{eqnarray}
    y_t(\ell,t) = \sum_{n=1}^{R} a_n e^{-\alpha_n t} \sin (n k \ell) \sin(n \omega t) + \eta(\ell,t) 
    \label{eqn:damped_wave_time} \\
     y_s(\ell,t) = \sum_{n=1}^{R} a_n e^{-\beta_n \ell} \sin (n k \ell) \sin(n \omega t) + \eta(\ell,t)
    \label{eqn:damped_wave_space}
\end{eqnarray}

Though this does not exactly satisfy the wave equation, waves in practice often have a damped sinusoidal nature.  We first consider the model in \eqref{eqn:damped_wave_time} (damping in time) where we run our algorithm with $\Y \in \mathbb{R}^{1000 \times 4000}$, $R=10$, $k = 2 \pi$, $\omega = 12 \pi$, $\alpha_n = n$, $a_n$ roughly of the order of $10^{1}$, and additive white Gaussian noise $\eta(\ell,t)$ of variance $\approx 10$ and $0$ mean (for the noisy case). We set $\gamma = 10^9$ and $\lambda = 200$ (see supplement for the behaviour of the algorithm with respect to changes in $\gamma$ and $\lambda$).  From wave-informed matrix factorization, when there is damping in time, we still obtain the columns of $\D$ (vibrations in space, which is undamped) as clear underlying sinusoids (Figure \ref{fig:Merge_4}).  We emphasize here that the modes of vibration are exactly recovered from data even in the presence of large amounts of noise (Figure \ref{fig:Merge_4} (left)), whereas low-rank matrix factorization does not obtain pure sinusoids even in the noiseless case (Figure \ref{fig:Merge_4} (right) -- the dotted blue curves change in amplitude over space).  To quantify this performance, in the appendix we provide the error between modes recovered by our algorithm and the ground-truth modes. Recall our algorithm also provides guarantees of polynomial-time global optimality (unlike the work of \cite{Tetali2019}), without using a library matrix as mentioned in \cite{lai2020full}. 
 Next, we run wave-informed matrix factorization on the model in \eqref{eqn:damped_wave_space} (damping in space) 
 with the same parameters as above and $\beta_n = n/2$ (also including the additive noise term) and visually compare our algorithms to other modal analysis methods (\Ai, \B, \C, \Dc, \E, \F, \G) (Figure~\ref{fig:damped_sines}). Here we show recovering damped sinusoids is possible under heavy noise (see two left columns of $\mathbf{D}$ demonstrated in Figure~\ref{fig:damped_sines}), where we have chosen the most significant columns of each method. A close look at the recovered modes indicates much cleaner recovery of damped (in space) sinusoids for our method (\Ai) compared to others.  For example, (\B) is the closest in performance to our method, but distortions can be observed in the tails of the components.

\begin{figure}[ht]
\vspace{-4mm}
    \centering
    \includegraphics[trim=10 10 10 150, clip, scale=0.08]{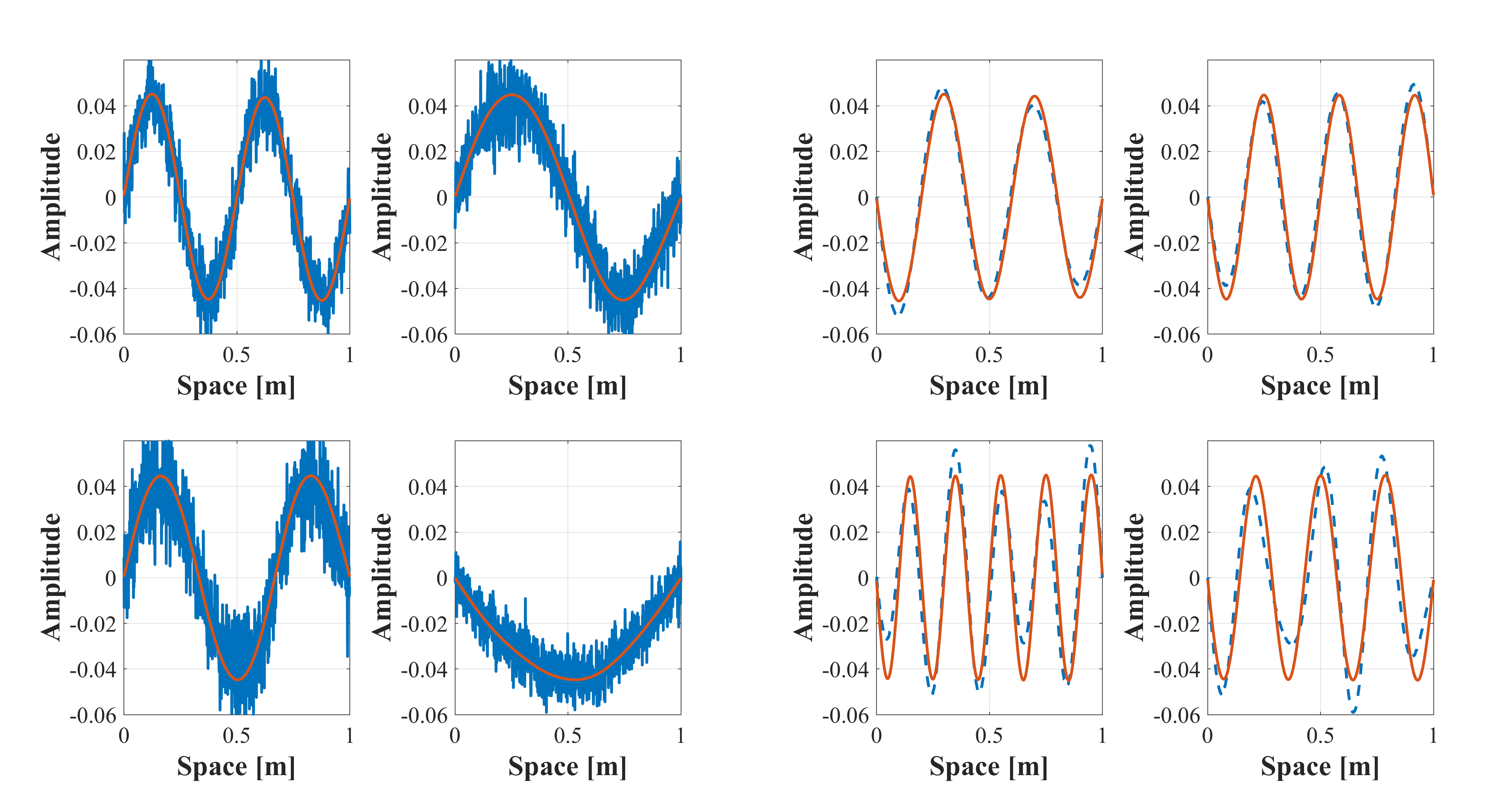}
    \caption{4 columns of $\D$ from low-rank matrix factorization (blue) and wave-informed matrix factorization (red) with the temporally damped vibrating string model.  Showing with noisy data (left) and the noiseless case (right). \vspace{-5mm}}
    \label{fig:Merge_4}
\end{figure}
\begin{figure}[ht]
    \centering
    \vspace{-4mm}
    \includegraphics[trim=10 10 10 15, clip, width=\linewidth]{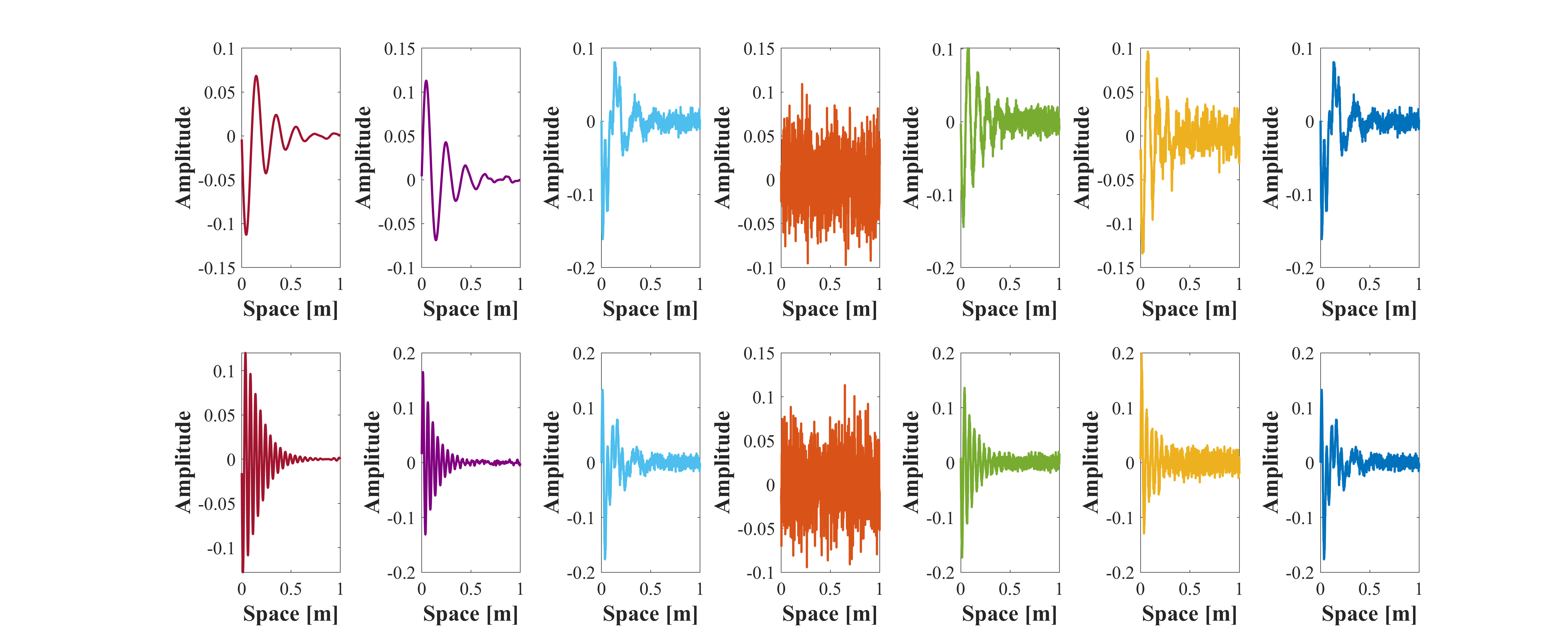}
    \caption{Two recovered modes (rows) of spatially damped sinusoids (\Ai), (\B), (\C), (\Dc), (\E), (\F), (\G).}    \label{fig:damped_sines}
    \vspace{-5mm}
\end{figure}

\myparagraph{Conclusions}
We have developed a framework for a wave-informed matrix factorization algorithm with provable, polynomial-time global optimally guarantees. Output from the algorithm was compared with that of low-rank matrix factorization and state-of-the-art algorithms for modal and component analysis. We demonstrated that the wave-informed approach learns representations that are more physically relevant and practical for the purpose of material characterization and modal analysis.  Future work will include 1) generalizing this approach to a variety of linear PDEs beyond the wave equation as well as wave propagation along more than one dimension, 2) applications in baseline-free anomaly detection for structural health monitoring \cite{alguri2018baseline, alguri2021sim}. 

\textbf{Acknowledgments} This work was partially supported by NIH NIA 1R01AG067396, ARO MURI W911NF-17-1-0304, NSF-Simons MoDL 2031985, and the National Science Foundation under award number ECCS-1839704.

\bibliographystyle{IEEEtran}
\bibliography{main}

\clearpage


\section{Supplementary Material}

Following is an example Laplacian matrix ($\L$),
\begin{eqnarray}
    \mathbf{L} &=& \frac{1}{(\Delta l)^2} \begin{bmatrix}
    -2 & 1 & 0 & 0 & 0 &\cdots & 0 \\
    1 & -2 & 1 & 0 & 0 & \cdots & 0 \\
    0 & 1 & -2 & 1 & 0 & \cdots & 0 \\
    \vdots & \vdots & \vdots & \vdots & \vdots & \ddots & \vdots \\
    0 & 0 & 0 & 0 & 0 & \cdots & -2 \\
    \end{bmatrix}.
    \label{eqn:Lap_mat}
\end{eqnarray}
To reduce complexity, all theorems and proofs consider $\Delta l = 1$ (for $\Delta l$ defined in equation \eqref{eqn:Lap_mat} without any loss of generality. The only change that needs to be accommodated is on the values in $\k$ obtained from equation (\ref{eq:k_max}) (i.e. while solving the polar). Observing equations (\ref{eqn:discrete_wave_eqn}) and (\ref{eqn:Lap_mat}), the only modification that is needed is to note that for $\Delta l \neq 1$ we follow Algorithm \ref{algoblock:meta-algo} and replace the final vector $\k$ with $\k \Delta l$.

\subsection{Algorithm Details}

\paragraph{Matrix Completion}
We note that our algorithm (and guarantees of polynomial time solutions) is easily generalized to any loss function $\ell(\D \X^\top)$, provided the loss function is once differentiable and convex w.r.t. $\D \X^\top$.  For example, the following is the modified algorithm for the matrix completion formulation in \eqref{eq:main_obj_missing}:

\begin{algorithm}
\caption{\bf{Meta-algorithm}}
\begin{algorithmic}[1]
\State Input $\mathbf{D}_{init}$, $\mathbf{X}_{init}$, $\mathbf{k}_{init}$ 
\State Initialize $ \left( \mathbf{D}, \mathbf{X}, \mathbf{k} \right) \leftarrow \left( \mathbf{D}_{init}, \mathbf{X}_{init}, \mathbf{k}_{init} \right)$
\While {global convergence criteria is not met}
\State Perform gradient descent on \eqref{eq:main_obj} with $N$ fixed to reach a first order stationary point $(\Tilde{\mathbf{D}}, \Tilde{\mathbf{X}}, \Tilde{\mathbf{k}})$\label{step:grad_desc}
\State Calculate the value of $\Omega^\circ_\theta \left(\tfrac{1}{\lambda}\left(\mathcal{A}^*\left(\Y-\tilde \D \tilde \X^\top\right)\right)\right)$ via Theorem \ref{thm:polar} and obtain $\mathbf{d}^*, \mathbf{x}^*, k^*$
\If {value of polar $\Omega^\circ_\theta \left(\tfrac{1}{\lambda}\left(\mathcal{A}^*\left(\Y-\tilde \D \tilde \X^\top\right)\right)\right) = 1$} \State {Algorithm converged to global minimum}
\Else
\State {Append  $(\mathbf{d}^*, \mathbf{x}^*, k^*)$ to $(\Tilde{\mathbf{D}}, \Tilde{\mathbf{X}}, \Tilde{\mathbf{k}})$ and update $\left( \mathbf{D}, \mathbf{X}, \mathbf{k} \right)$
\State $ \left( \mathbf{D}, \mathbf{X}, \mathbf{k} \right) \! \leftarrow \! \left( \left[ \tilde{\mathbf{D}}, \ \tau_{\mathcal{A}} \mathbf{d}^*  \right], \left[ \tilde{\mathbf{X}}, \ \tau_{\mathcal{A}} \mathbf{x}^*  \right] , \left[ \tilde{\mathbf{k}}^\top, \ k^*  \right]^\top
\right)$, $\tau_{\mathcal{A}}\! > \! 0$ is step size.}\label{step:step_size_mising}
\EndIf
\State Continue loop.
\EndWhile
\end{algorithmic}
\label{algoblock:meta-algo_missing}
\end{algorithm}
where $\mathcal{A}^*$ denotes the adjoint of the linear masking operator.

\paragraph{Gradient update equations} We set the derivative of the right hand side of equation \ref{eqn:obj_final} with respect to $\D$, $\X$ and $\k$ and utilize block coordinate descent of Gauss-Seidel type \cite{xu2013block} to reach a first order stationary point mentioned in step \ref{step:grad_desc} of Algorithm \ref{algoblock:meta-algo}. The number of columns of $\D$ and $\X$ (denoted by $N$ in equation (\ref{eq:main_obj})) is fixed and does not change during this step. The following are the gradient update equations (for stepsize $\alpha_i$):

\begin{gather}
    \D_j \leftarrow \D_j  - \alpha_j \left(  \left(\D \X^{\top} - \Y \right) \X_j  +  \lambda \D_j + 2 \gamma \lambda \left( \L + k^2_j \I  \right)^2 \D_j \right) \label{eq:grad_D} \\
    \X_j \leftarrow \X_j - \alpha_j \left( \left( \D \X^{\top} - \Y \right)^{\top} \D_j \right) \label{eq:grad_X} \\
    k_j \leftarrow \sqrt{-\frac{ \D_j^\top \L \D_j } {\|\D_j\|_2^2}} \label{eq:grad_k} 
\end{gather}

\paragraph{Gradient update equations for the Matrix completion setting}

\begin{gather}
    \D_j \leftarrow \D_j  - \alpha_j \left(  \mathcal{A}^* \left(\D \X^{\top} - \Y \right) \X_j  +  \lambda \D_j + 2 \gamma \lambda \left( \L + k^2_j \I  \right)^2 \D_j \right) \label{eq:grad_D} \\
    \X_j \leftarrow \X_j - \alpha_j \left( \mathcal{A}^* \left( \D \X^{\top} - \Y \right)^{\top} \D_j \right) \label{eq:grad_X} \\
    k_j \leftarrow \sqrt{-\frac{ \D_j^\top \L \D_j } {\|\D_j\|_2^2}} \label{eq:grad_k} 
\end{gather}

\paragraph{Step size computation} The step size, $\tau$ mentioned in Step 10 of Algorithm \ref{algoblock:meta-algo} is computed through the following quadratic minimization problem.
\begin{eqnarray}
    \min_{\substack{\tau \in \Re}} \tfrac{1}{2}\|\Y- \D_\tau \X_\tau^\top\|_F^2 + \frac{\lambda}{2} \sum_{i=1}^{N+1} \left(\|\left(\X_{\tau}\right)_i\|_F^2 + \| \left( \D_{\tau} \right)_i \|_F^2 + \gamma \|\L \left( \D_{\tau} \right)_i + \left(\k_{\tau}^2\right)_i  \left( \D_{\tau} \right)_i \|_F^2 \right)
    \label{eqn:for_tau}
\end{eqnarray}

where, $\left(\D_\tau, \X_\tau, \k_{\tau} \right) = \left( \left[ \tilde{\mathbf{D}}, \ \tau \mathbf{d}^*  \right], \left[ \tilde{\mathbf{X}}, \ \tau \mathbf{x}^*  \right] , \left[ \tilde{\mathbf{k}}^\top, \ k^*  \right]^\top
\right)$ and $N$ is the number of columns in $\tilde{\D}$.

\paragraph{Proposition} The optimal step size $\tau^*$ that minimizes the expression in (\ref{eqn:for_tau}) is given by:
\begin{eqnarray}
    \tau^* = \frac{\sqrt{(\d^*)^\top \left( \Y - \tilde{\D} \tilde{\X}^{\top} \right) \x^* - \lambda }}{\|\d^*\|_2 \|\x^*\|_2 }
    \label{eq:for_tau_}
\end{eqnarray}

\begin{proof}
Let $f(\tau)$ represent the objective function in \eqref{eqn:for_tau} with everything except for $\tau$ held fixed:
\begin{eqnarray}
    f(\tau) = \tfrac{1}{2}\|\Y- \D_\tau \X_\tau^\top\|_F^2 + \tfrac{\lambda}{2} \sum_{i=1}^{N+1} \left(\|\left(\X_{\tau}\right)_i\|_F^2 + \| \left( \D_{\tau} \right)_i \|_F^2 + \gamma \|\L \left( \D_{\tau} \right)_i + \left(\k_{\tau}\right)^2_i  \left( \D_{\tau} \right)_i \|_F^2 \right)
\end{eqnarray}

Observe that from the solution to the Polar problem we have by construction that the regularization term $\theta(\d^*,\x^*)=1$, so combined with the positive homogeneity of $\theta$ we have have that minimizing $f(\tau)$ w.r.t. $\tau$ is equivalent to solving:
%
\begin{equation}
\min_{\tau \geq 0} \tfrac{1}{2}\|\Y- \tilde \D \tilde \X^\top - \tau^2 \d^* (\x^*)^\top \|_F^2 + \lambda \tau^2
\end{equation}

Taking the gradient of the above w.r.t. $\tau^2$ and solving for 0 gives:
\begin{equation}
(\tau^*)^2 = \frac{\langle \Y-\tilde \D \tilde \X^\top, \d^* (\x^*)^\top \rangle - \lambda}{ \|\d^* (\x^*)^\top\|_F^2}
\end{equation}
The result is completed by noting that the numerator is guaranteed to be strictly positive due to the fact that the Polar solution has value strictly greater that 1.
\end{proof}

The step size, $\tau_{\mathcal{A}}$ for the matrix completion algorithm, mentioned in Step 10 of Algorithm \ref{algoblock:meta-algo_missing} is computed through the following quadratic minimization problem.
\begin{eqnarray}
    \min_{\substack{\tau_{\mathcal{A}} \in \Re}} \tfrac{1}{2}\|\Y- \D_{\tau_{\mathcal{A}}} \X_{\tau_{\mathcal{A}}}^\top\|_F^2 + \frac{\lambda}{2} \sum_{i=1}^{N+1} \left(\|\left(\X_{\tau_{\mathcal{A}}}\right)_i\|_F^2 + \| \left( \D_{\tau_{\mathcal{A}}} \right)_i \|_F^2 + \gamma \|\L \left( \D_{\tau_{\mathcal{A}}} \right)_i + \left(\k_{\tau_{\mathcal{A}}}^2\right)_i  \left( \D_{\tau_{\mathcal{A}}} \right)_i \|_F^2 \right)
    \label{eqn:for_tau_missing}
\end{eqnarray}

where, $\left(\D_{\tau_{\mathcal{A}}}, \X_{\tau_{\mathcal{A}}}, \k_{\tau_{\mathcal{A}}} \right) = \left( \left[ \tilde{\mathbf{D}}, \ {\tau_{\mathcal{A}}} \mathbf{d}^*  \right], \left[ \tilde{\mathbf{X}}, \ {\tau_{\mathcal{A}}} \mathbf{x}^*  \right] , \left[ \tilde{\mathbf{k}}^\top, \ k^*  \right]^\top
\right)$ and $N$ is the number of columns in $\tilde{\D}$.

\paragraph{Proposition} The optimal step size $\tau^*_{\mathcal{A}}$ that minimizes the expression in (\ref{eqn:for_tau_missing}) is given by:
\begin{eqnarray}
\tau^*_{\mathcal{A}^*} = \frac{ \sqrt{ \langle \mathcal{A}^* \left( \Y-\tilde \D \tilde \X^\top \right), \mathcal{A}^* \left( \d^* {\x^*}^\top \right)\rangle - \lambda }}{ \| \mathcal{A}^* \left( \d^* {\x^*}^\top \right) \|_F}
\end{eqnarray}
\begin{proof}
Following the same steps of the previous proof, we obtain:

\begin{equation}
\min_{\tau_{\mathcal{A}} \geq 0} \tfrac{1}{2}\| \mathcal{A}^* \left( \Y- \tilde \D \tilde \X^\top \right) - \tau_{\mathcal{A}}^2 \mathcal{A}^* \left( \d^* (\x^*)^\top \right) \|_F^2 + \lambda \tau_{\mathcal{A}}^2
\end{equation}

The proof is completed by minimizing the quadratic for $\left(\tau_{\mathcal{A}}\right)^2$.

\end{proof}

\subsection{Proofs}

{
\renewcommand{\theproposition}{\ref{prop:equiv_prob}}
\begin{proposition}
The optimization problem in \eqref{eq:main_obj} is a special case of the problem considered in \cite{haeffele2019structured}.
\end{proposition}
\addtocounter{proposition}{-1}
}

\begin{proof}
The problem considered in \cite{haeffele2019structured} is stated in \eqref{eq:gen_obj}. Comparing it with (\ref{eq:main_obj}) we have:
\begin{gather}
    \ell \left( \D\X^{\top} \right) = \tfrac{1}{2}\| \Y - \D\X^{\top} \|_F^2\\
    \bar{\theta} \left( \D_i, \X_i \right) = \tfrac{1}{2} \left( \| \D_i \|^2_2 + \| \X_i \|^2_2 + \gamma \min_{k_i}  \| \L \D_i - k_i \D_i \|^2_2 \right)
\end{gather}
Observe that $\ell ( \hat{\Y} ) = \tfrac{1}{2}\| \Y - \hat{\Y} \|^2_F$ is clearly convex and differentiable w.r.t $\hat \Y$.  
 Now to realize that the optimization problem in \eqref{eq:main_obj} is a special case of the problem considered in \cite{haeffele2019structured}, it suffices to check that $\bar \theta(\d,\x)$ satisfies the three conditions of a rank-1 regularizer from \cite{haeffele2019structured}.

\begin{enumerate}
    \item $\bar \theta(\alpha \d, \alpha \x) = \alpha^2 \theta(\d, \x), \ \forall (\d,\x)$ and $\forall \alpha \geq 0$.
    
    For any $\alpha > 0$, $\forall (\d, \x)$ :
    \begin{equation}
         \begin{split}
        \bar{\theta}\left( \alpha \d, \alpha \x \right)  &= \| \alpha \d \|^2_2 + \| \alpha \x \|^2_2 + \gamma \min_{k} \| \alpha \L \d + k^2 \alpha \d \|^2_F \\
        &= \alpha^2 \| \d \|^2_2 + \alpha^2 \| \x \|^2_2 + \alpha^2 \gamma \min_{k} \| \L \d + k^2 \d \|^2_F \\
       &= \alpha^2 \bar{\theta}(\d,\x)
    \end{split}
    \end{equation}
    
    where we note that scaling $\d$ by $\alpha > 0$ does not change the optimal value of $k$ in the third term, allowing $\alpha^2$ to be moved outside of the norm.
    
    \item $\bar \theta(\d, \x) \geq 0, \ \forall (\d,\x)$.

    Clearly, all terms in $\bar{\theta}(\d,\x)$ are non-negative, thus, $\forall (\d,\x)$, we have $\bar{\theta}(\d,\x)\geq 0$.
    \item For all sequences $(\d^{(n)},\x^{(n)})$ such that $\|\d^{(n)}(\x^{(n)})^\top\| \rightarrow \infty$ then $\bar \theta(\d^{(n)},\x^{(n)}) \rightarrow \infty$.
    
    Here, note that the following is true for all $(\d,\x)$:
    \begin{equation}
    \|\d \x^\top \|_F = \|\d\|_2 \|\x\|_2 \leq \tfrac{1}{2}(\|\d\|_2^2 + \|\x\|_2^2) \leq \tfrac{1}{2}\left(\|\d\|_2^2 + \|\x\|_2^2 + \min_{k} \| \L \d + k^2 \d \|^2_2\right)
    \end{equation}
    As a result we have $\forall (\d, \x)$ that $\|\d\x^\top\|_F \leq \bar \theta(\d,\x)$, completing the result.
\end{enumerate}
\end{proof}

Before proving Theorem \ref{thm:polar} we first prove an intermediate lemma regarding Lipschitz constants.
\begin{lemma}
\label{lem:L_bound}
Given a set of constants $\lambda_i, \ i=1, 2, \ldots$ such that $\forall i$, $\mu_\lambda \leq \lambda_i \leq 0$, and a constant $\gamma > 0$, let the function $f$ be defined as:
\begin{equation}
f(x) = \max_i \frac{1}{\sqrt{1+\gamma(x + \lambda_i)^2}}.
\end{equation}
Then, over the domain $0 \leq x \leq \mu_x$ $f$ is Lipschitz continuous with Lipschitz constant $L_f$ bounded as follows:
\begin{equation}
L_f \leq \left[\begin{cases} \frac{2}{3\sqrt{3}} \sqrt{\gamma} & \gamma\geq \frac{1}{2 \max \{ \mu_\lambda^2, \mu_x^2 \}} \\
\gamma \max \{ -\mu_\lambda, \mu_x\} (1+\gamma \max \{\mu_\lambda^2, \mu_x^2\} )^{-\tfrac{3}{2}} & \mathrm{otherwise}
\end{cases}  \right] \leq \frac{2}{3 \sqrt{3}}\sqrt{\gamma}. 
\end{equation}
\end{lemma}

\begin{proof}
First, note that for any two Lipschitz continuous functions $\psi_a$ and $\psi_b$, with associated Lipschitz constants $L_a$ and $L_b$, respectively, one has that the point-wise maximum of the two functions, $\psi(x) = \max \{ \psi_a(x), \psi_b(x) \}$, is also Lipschitz continuous with Lipschitz constant bounded by $\max \{ L_a, L_b \}$.  This can be easily seen by the following two inequalities:
\begin{align}
&\psi_a(x') \leq \psi_a(x) + | \psi_a(x') - \psi_a(x) | \leq \psi(x) + | \psi_a(x') - \psi_a(x) | \leq \psi(x) + L_a |x' - x| \\
&\psi_b(x') \leq \psi_b(x) + | \psi_b(x') - \psi_b(x) | \leq \psi(x) + | \psi_b(x') - \psi_b(x) | \leq \psi(x) + L_b |x' - x|
\end{align}
From this we have:
\begin{equation}
\begin{split}
&\psi(x') = \max \{ \psi_a(x'), \psi_b(x') \} \leq \max \{ \psi(x) + L_a |x'-x|, \psi(x) + L_b |x'-x| \} = \psi(x) + \max \{L_a, L_b \} |x'-x| \\
&\implies \psi(x') - \psi(x) \leq \max \{ L_a, L_b \} |x'-x|
\end{split}
\end{equation}
which implies the claim from symmetry.

Now, if we define the functions $g$ and $h$ as
\begin{equation}
g(x, \lambda) = \frac{1}{\sqrt{1+\gamma(x + \lambda)^2}} \ \ \ \ \ \ \ h(b) = \gamma b (1+\gamma b^2)^{-\tfrac{3}{2}}
\end{equation}
we have that the Lipschitz constant of $f$, denoted as $L_f$, is bounded as:
\begin{equation}
\label{eq:h_max}
\begin{split}
L_f \leq &\max_i \sup_{x \in [0, u_x]} \left| \frac{\partial}{\partial x} g(x, \lambda_i) \right| \leq \sup_{\lambda \in [\mu_\lambda,0]} \sup_{x \in [0, \mu_x]} \left| \frac{\partial}{\partial x} g(x, \lambda) \right| =  \\
&\sup_{\lambda \in [\mu_\lambda,0]} \sup_{x \in [0, \mu_x]} \left| - \frac{\gamma(x+\lambda)}{ (\sqrt{1+\gamma (x+\lambda)^2})^3 } \right|
= \sup_{b \in [\mu_\lambda, \mu_x]} \left| h(b) \right| = \sup_{b \in [0, \max \{ -\mu_\lambda, \mu_x \} ]} h(b)
\end{split}
\end{equation}
Where the first inequality is from the result above about the Lipschitz constant of the point-wise maximum of two functions and the simple fact that the Lipschitz constant of a function is bounded by the maximum magnitude of its gradient, and the final equality is due to the symmetry of $|h(b)|$ about the origin.  Now, finding the critical points of $h(b)$ for non-negative $b$ we have:
\begin{equation}
\begin{split}
h'(b) = \gamma (1+\gamma b^2)^{-\tfrac{3}{2}} - 3 \gamma^2 b^2 (1+\gamma b^2)^{- \tfrac{5}{2}} = 0 
\implies 3 \gamma b^2 = (1+ \gamma b^2) \implies b^* = \frac{1}{\sqrt{2 \gamma}}
\end{split} 
\end{equation}
Note that $h(0)=0$, $h(b)>0$ for all $b>0$, and $h'(b)<0$ for all $b > b^*$.  As a result, $b^*$ will be a maximizer of $h(b)$ if it is feasible, otherwise the maximum will occur at the extreme point $b = \max \{-\mu_\lambda, \mu_x \}$.  From this we have the result:
\begin{equation}
\begin{split}
L_f &\leq \begin{cases} h(b^*) = \frac{2}{3 \sqrt{3}} \sqrt{\gamma} & \gamma\geq \frac{1}{2 \max \{ \mu_\lambda^2, \mu_x^2\}} \\
h(\max \{-\mu_\lambda, \mu_x \} ) = \gamma \max \{ -\mu_\lambda, \mu_x \} (1+\gamma \max \{ \mu_\lambda^2, \mu_x^2 \})^{-\tfrac{3}{2}} & \text{otherwise}
\end{cases} \\
& \leq h(b^*) = \frac{2}{3\sqrt{3}}\sqrt{\gamma}.
\end{split}
\end{equation}
\end{proof}

{
\renewcommand{\thetheorem}{\ref{thm:polar}}

\begin{theorem}
For the objective in \eqref{eq:main_obj}, the associated polar problem is equivalent to:
\begin{equation}
\Omega_\theta^\circ (\Z) = \max_{\d \in \Re^{L}, \x \in \Re^{T}, k\in \Re} \d^\top \Z \x \st \|\d\|^2_F + \gamma \|\L\d + k^2 \d\|_F^2 \leq 1, \ \|\x\|^2_F \leq 1, \ 0 \leq k \leq 2.
\end{equation}
Further, let $\L = \Gamma \Lambda \Gamma^\top$ be an eigen-decomposition of $\L$ and define the matrix $\A( \bar k) = \Gamma(\I + \gamma(\bar k \I + \Lambda)^2) \Gamma^\top$.  Then, if we define $\bar k^*$ as%
\begin{equation}
\bar k^* = \argmax_{\bar k \in [0,4]} \| \A(\bar k)^{-1/2} \Z \|_2
\label{eq:k_max}
\end{equation}
optimal values of $\d,\x, k$ are given as $\d^* = \A(\bar k^*)^{-1/2} \bar \d$, $\x^* = \bar \x$, and $k^* = (\bar k^*)^{1/2}$ where $(\bar \d, \bar \x)$ are the left and right singular vectors, respectively, associated with the largest singular value of $\A(\bar k^*)^{-1/2} \Z$.  Additionally, the above line search over $\bar k$ is Lipschitz continuous with a Lipschitz constant, $L_{\bar k}$, which is bounded by:
\begin{equation}
L_{\bar k} \leq \left[ \begin{cases} \frac{2}{3 \sqrt{3}} \sqrt{\gamma} \|\Z\|_2 & \gamma \geq \frac{1}{32} \\ 
4 \gamma (1 + 16 \gamma)^{-\tfrac{3}{2}} \|\Z\|_2 & \mathrm{otherwise} \end{cases} \right] \leq \frac{2}{3 \sqrt{3}} \sqrt{\gamma} \|\Z\|_2
\end{equation}
\end{theorem}
\addtocounter{theorem}{-1}
}

\begin{proof}
The polar problem associated with the objectives of the form \eqref{eq:gen_obj} as given in \cite{haeffele2019structured} is:

\begin{gather}
       \Omega_\theta^\circ (\Z) = \sup_{\d,\x} \d^{\top} \Z \x \st \bar \theta(\d,\x) \leq 1 
\end{gather}

For our particular problem, due to the bilinearity between $\d$ and $\x$ in the objective the above is equivalent to:
\begin{gather}
         \Omega_\theta^\circ (\Z)   = \sup_{\d,\x} \d^{\top} \Z \x \st \| \x \|^2_2 \leq 1, \|\d\|^2_2 + \gamma \min_{k} \| \L \d + k^2 \d \|^2_2 \leq 1
\end{gather}
Note that this is equivalent to moving the minimization w.r.t. $k$ in the regularization constraint to a maximization over $k$:
 \begin{gather}
         \Omega_\theta^\circ (\Z)   = \sup_{\d,\x,k} \d^{\top} \Z \x \st \| \x \|^2_2 \leq 1, \|\d\|^2_2 + \gamma \| \L \d + k^2 \d \|^2_2 \leq 1
\end{gather}

Next, note that maximizing w.r.t. $\d$ while holding $\x$ and $k$ fixed is equivalent to solving a problem of the form:
\begin{align}
\max_\d \langle \d, \Z \x \rangle \ST \d^\top \A \d \leq 1
\end{align}
for some positive definite matrix $\A$.  If we make the change of variables $\bar \d = \A^{1/2} \d$, this then becomes:
\begin{align}
\max_{\bar \d} \ \ &\langle \bar \d, \A^{-1/2} \Z \x \rangle \ST \| \bar \d \|_2^2 \leq 1 = \| \A^{-1/2} \Z \x \|_2 
\end{align}
where the optimal $\bar \d$ and $\d$ are obtained at 
\begin{align}
\bar \d_{opt} &= \frac{\A^{-1/2} \Z \x}{\|\A^{-1/2} \Z \x\|_2} \\
\label{eq:d_opt}
\d_{opt} &= \A^{-1/2} \bar \d_{opt} = \frac{\A^{-1} \Z \x}{\|\A^{-1/2} \Z \x\|_2}
\end{align}
For our particular problem, if we make the change of variables $\bar k = k^2$ we have that $\A$ is given by:
\begin{equation}
\A(\bar k) = (1+\gamma \bar k^2) \I + \gamma \L^2 + 2 \bar k \gamma \L
\end{equation}
where we have used that $\L$ is a symmetric matrix.  If we let $\L = \Gamma \Lambda \Gamma^\top$ be an eigen-decomposition of $\L$ then we can also represent $\A(\bar k)$ and $\A(\bar k)^{-1/2}$ as:
\begin{align}
\A(\bar k) &= \Gamma \left( (1+\gamma \bar k^2) \I + \gamma \Lambda^2 + 2 \bar k \gamma \Lambda \right) \Gamma^\top \\
 &= \Gamma (\I + \gamma ( \bar k \I +  \Lambda)^2) \Gamma^\top \\ 
 \label{eq:A12}
\A(\bar k)^{-1/2} &= \Gamma (\I + \gamma ( \bar k \I +  \Lambda)^2)^{-1/2} \Gamma^\top 
\end{align}
Now if we substitute back into the original polar problem, we have:
\begin{align}
\Omega^\circ(\Z) &= \max_{\x, \bar k} \| \A(\bar k)^{-1/2} \Z \x \|_2 \ST \|\x\|_2^2 \leq 1 \\
\label{eq:k_line}
&= \max_{\bar k} \|(\A(\bar k)^{-1/2} \Z \|_2
\end{align}
where $\|\cdot\|_2$ denotes the spectral norm (maximum singular value).  Similarly, for a given $\bar k$ the optimal $\x$ is given as the right singular vector of $\A(\bar k)^{-1/2}\Z$ associated with the largest singular value.

As a result, we can solve the polar by performing a line search over $\bar k$, then once an optimal $\bar k^*$ is found we get $\x^*$ as the largest right singular vector of $\A(\bar k^*)^{-1/2} \Z$ and the optimal $\d^*$ from \eqref{eq:d_opt} (where $\d_{opt}$ will be the largest left singular vector of $\A(\bar k^*)^{-1/2} \Z$ multiplied by $\A(\bar k^*)^{-1/2}$).

Now, an upper bound for $\bar k$ can be calculated from the fact that the optimal $\bar k$ is defined using a minimization problem, i.e.
\begin{eqnarray}
    \bar k^* = \argmin_{\bar k} \| \L \d + \bar k \d \|^2
\end{eqnarray}
So we note for any $\d$,
\begin{eqnarray}
   \bar k^* = -\frac{ \d^\top \L \d} {\|\d\|_2^2}
   \label{eq:opt_k}
\end{eqnarray}
which is bounded by the smallest eigenvalue of $\L$ (note that $\L$ is negative (semi)definite). We cite the literature on eigenvalues of discrete second derivatives \cite{chung2000discrete} to note that all eigenvalues of $\L$ (irrespective of the boundary conditions) lie in the range $[\frac{-4}{\Delta l},0]$, since we specifically chose $\Delta l = 1$ (without loss of generality), we have that all eigenvalues of $\L$ lie in the range$[-4,0]$.






As a result, we need to only consider $\bar k$ in the range $[0,4]$:
\begin{equation}
\bar k^* = \argmax_{\bar k \in [0,4]} \| \A(\bar k)^{-1/2} \Z \|_2
\end{equation}

Finally, to show the Lipschitz continuity, we define the function:
\begin{equation}
f_\A(\bar k) = \| \A(\bar k)^{-1/2} \|_2
\end{equation}
and then note the following:
\begin{equation}
\begin{split}
& \left| \|\A(\bar k)^{-1/2}\Z \|_2 - \| \A(\bar k ')^{-1/2} \Z \|_2 \right| \\
\leq & \left\| \left( \A(\bar k)^{-1/2} - \A(\bar k')^{-1/2} \right) \Z \right\|_2 \\
\leq & \left\| \A(\bar k)^{-1/2} - \A(\bar k')^{-1/2} \right\|_2 \left\| \Z \right\|_2 \\
\leq & L_\A |\bar k - \bar k'| \|\Z\|_2
\end{split}
\end{equation}
where the first inequality is simply the reverse triangle inequality, the second inequality is due to the spectral norm being submultiplicative, and $L_\A$ denotes the Lipschitz constant of $f_\A(\bar k)$.  From the form of $\A(\bar k)^{-1/2}$ in \eqref{eq:A12} note that we have:
\begin{equation}
f_\A(\bar k) \equiv \| \A(\bar k)^{-1/2}\|_2 = \max_i \frac{1}{\sqrt{1 + \gamma(\bar k + \Lambda_{i,i})^2}}
\end{equation}
so the result is completed by recalling from our discussion above that $\Lambda_{i,i} \in [-4, 0], \forall i$ and applying Lemma \ref{lem:L_bound}.

\end{proof}

{
\renewcommand{\thecorollary}{\ref{cor:poly_time}}
\begin{corollary}
Algorithm \ref{alg:meta} produces an optimal solution to \eqref{eq:main_obj} in polynomial time.
\end{corollary}
\addtocounter{corollary}{-1}
}
\begin{proof}
This result is largely a Corollary from Theorem \ref{thm:polar} and what is known in the literature.  Namely, the authors of \cite{xu2013block} show that the block coordinate update steps \eqref{eq:grad_D}, \eqref{eq:grad_X} and \eqref{eq:grad_k} in step \ref{step:grad_desc} of Algorithm \ref{algoblock:meta-algo} reaches a stationary point in polynomial time because the objective function \eqref{eq:main_obj} is convex w.r.t. each ($\D$, $\X$, $\k$) if the other terms are held fixed. Next, 
by Theorem \ref{thm:polar} the optimization problem for solving the polar can be done in polynomial time.  Finally, it has been noted in the literature on structured matrix factorization \cite{bach2013convex} that the polar update step is equivalent to a generalized conditional gradient step 
and if the conditional gradient steps (i.e., the polar problem) can be solved exactly (as we show in Theorem \ref{thm:polar}) then the algorithm converges in a polynomial number of such steps.  As a result, due to the fact that the block coordinate update steps reach a stationary in polynomial time we will perform a polar update step (a.k.a., a conditional gradient step) at polynomial time intervals, so the overall algorithm is also guaranteed to converge in polynomial time.

\end{proof} 

\subsection{Interpreting the Wave-Informed Regularizer as a Bandpass Filter}

Note that when identifying the optimal $k$ value in the polar program, we solve for
\begin{equation}
\label{eq:k_opt}
\argmax_{k \in [0,4]} \| \Gamma (\I + \gamma ( \bar k \I +  \Lambda)^2)^{-1/2} \Gamma^\top \Z \|_2 \; .
\end{equation}
This optimization has an intuitive interpretation from digital signal processing. Given that $\Gamma$ contains the eigenvectors of a Toeplitz matrix, those eigenvectors have spectral qualities similar to the discrete Fourier transform (the eigenvectors of a circulant matrix would be the discrete Fourier transform). As a result, $\Gamma^\top$ transforms the data $\Z$ into a spectral-like domain and $\Gamma$ returns the data back to the original domain. Since the other terms are all diagonal matrices, they represent element-wise multiplication across the data in the spectral domain. This is approximately equivalent to a filtering operation, with filter coefficients given by the diagonal entries of $(\I + \gamma ( \bar k \I +  \Lambda)^2)^{-1/2}$.

Furthermore, recall that the transfer function of a 1st-order Butterworth filter is given by:
\begin{equation}
T(\omega) = \frac{1}{\sqrt{1+\gamma (k_0 + \omega)^2}}
\end{equation}
where $k_0$ is the center frequency of the passband of the filter and $1/\sqrt{\gamma}$ corresponds to the filter's $-3$dB cut-off frequency.  Comparing this to the filter coefficients from \eqref{eq:k_opt} we note that the filter coefficients are identical to those of the 1st-order Butterworth filter, where $\Lambda$ corresponds to the angular frequencies. 
As a result, we can consider this optimization as determining the optimal filter center frequency $(\bar k)$ with fixed bandwidth $(1/\sqrt{\gamma})$  that retains the maximum amount of signal power from $\Z$.  Likewise the choice of the $\gamma$ hyperparameter sets the bandwidth of the filter.  As $\gamma \rightarrow \infty$, the filter bandwidth approaches 0 and thereby restricts us to a single-frequency (i.e., Fourier) solution. 
Furthermore, we can provide a recommended lower bound for $\gamma$ according to $\gamma > 1/k_{bw}^2$, where $k_{bw}$ is the bandwidth of the signal within this spectral-like domain.

\section{Additional Results}

\subsection{Characterizing Multi-Segment Transmission Lines}

 For the simulation considered in \textbf{Characterizing Multi-Segment Transmission Lines} in \S~\ref{sec:results}, Figure \ref{fig:Merge_2} shows two example columns for $\gamma = 50$ and $\lambda = 0.6$. We show in Figure \ref{fig:polar_objective} the reduction of objective value over iterations, the rate of change of objective value per iteration and the value of polar after each iteration of the overall meta-algorithm. Figures~\ref{fig:workflow_a}, \ref{fig:workflow_b}, \ref{fig:workflow_c} show curves similar to Figure~\ref{fig:polar_objective} but for different choices of regularization parameters.

\begin{figure}
    \centering
    \includegraphics[scale=0.75]{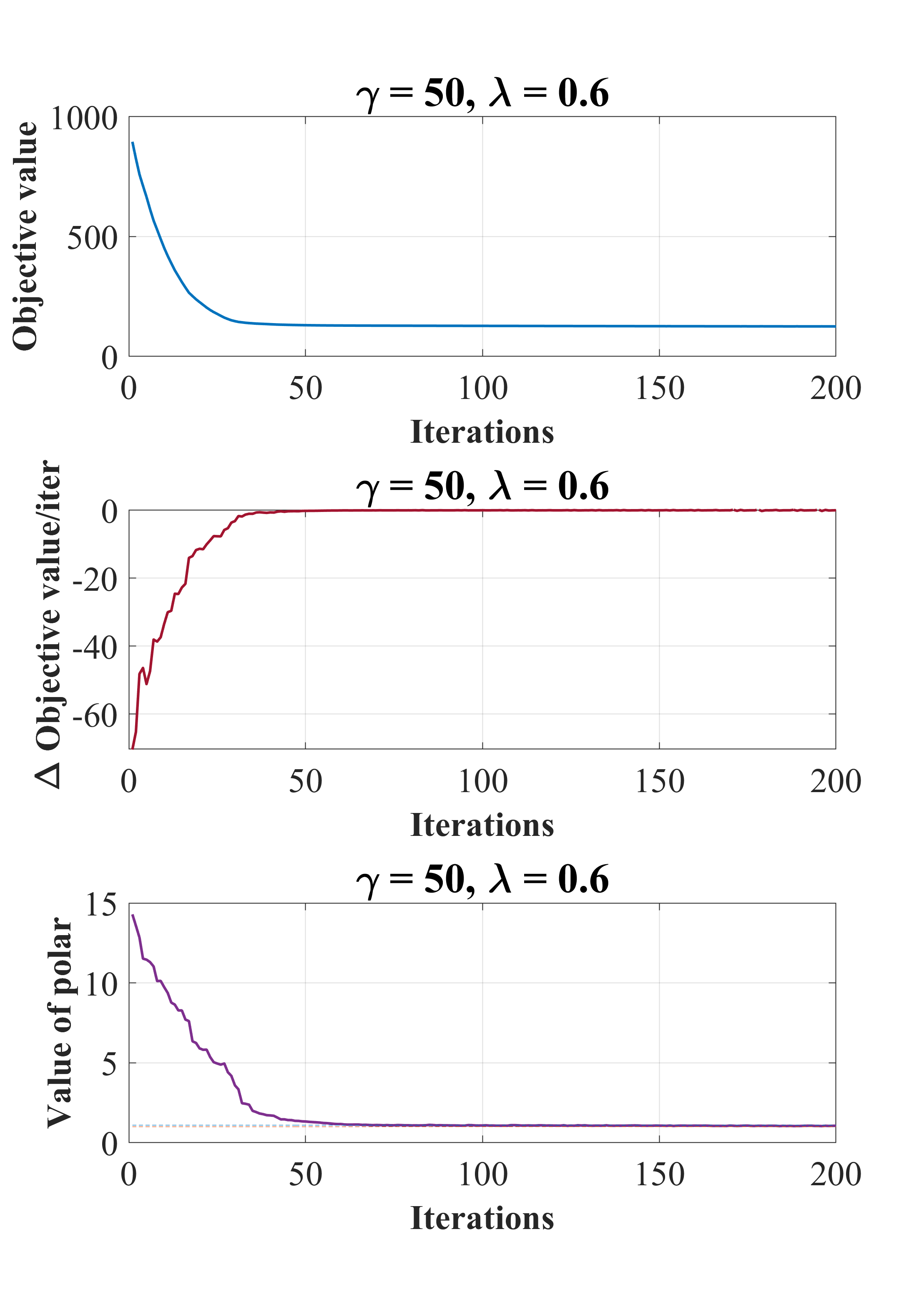}
    \caption{Objective value, rate of change of objective value per iteration and the value of polar after each iteration of the meta-algorithm for the simulation in subsection \textbf{Characterizing Multi-Segment Transmission Lines} \S \ref{sec:results}.}
    \label{fig:polar_objective}
\end{figure}

We observe in Figure \ref{fig:polar_objective} that the change of objective value is almost zero after 70 iterations. At this point, the polar value also goes below 1.1 (which as we note in the main paper also implies we are close to the global minimum) and eventually reaches 1, providing a certificate of global optimality. 
In practice, we often stop the algorithm after the polar value reaches below 1.1 as this guarantees a very close to optimal solution. 

\begin{figure}
    \centering
    \includegraphics[width=0.95\textwidth]{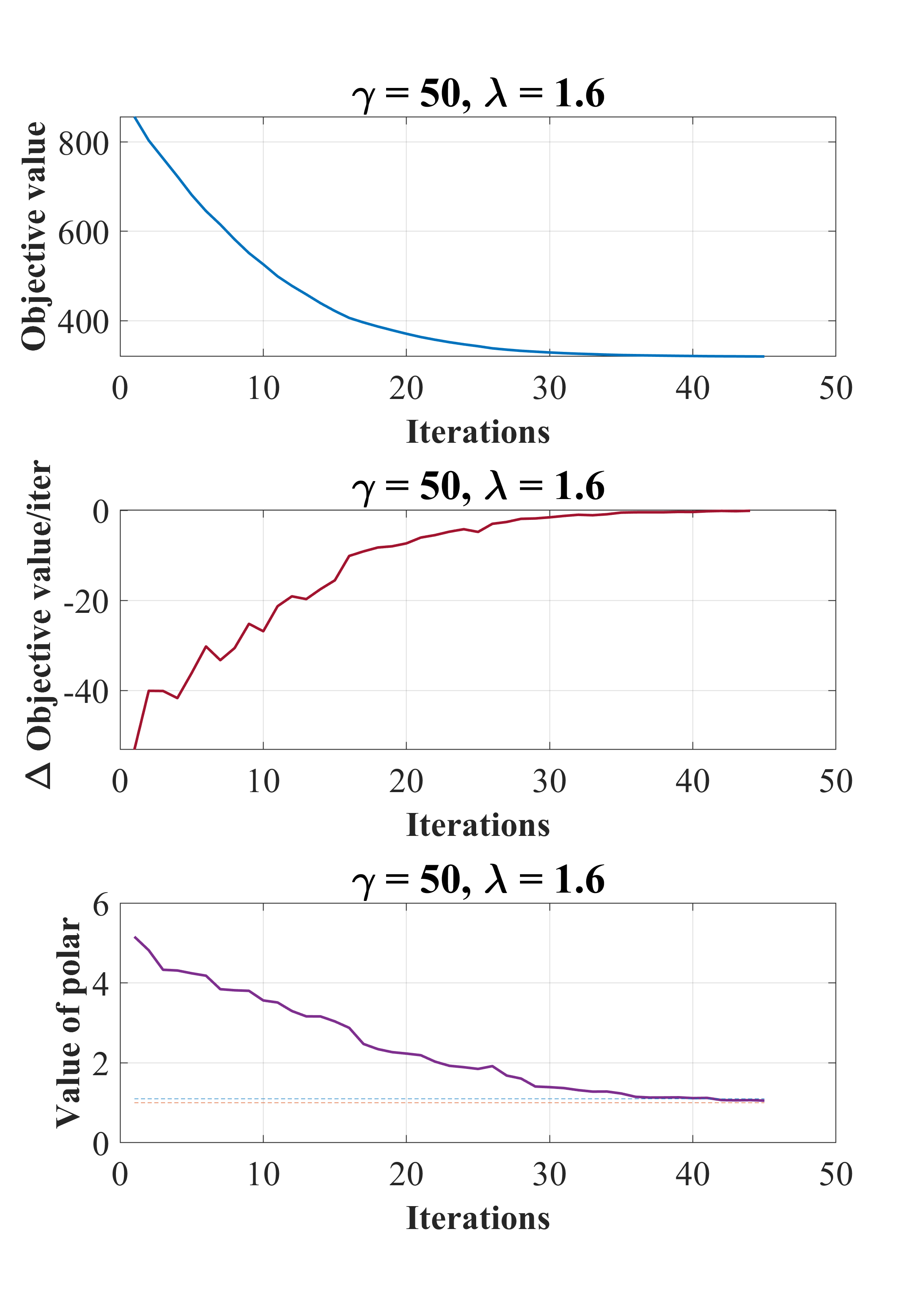}
    \caption{Objective value, rate of change of objective value per iteration and the value of polar after each iteration of the meta-algorithm for the simulation in \textbf{Characterizing Multi-Segment Transmission Lines} \S \ref{sec:results} with values of regularization parameters $\gamma = 50$, $\lambda = 1.6$}
    \label{fig:workflow_a}
\end{figure}
\begin{figure}
    \centering
    \includegraphics[width=0.95\textwidth]{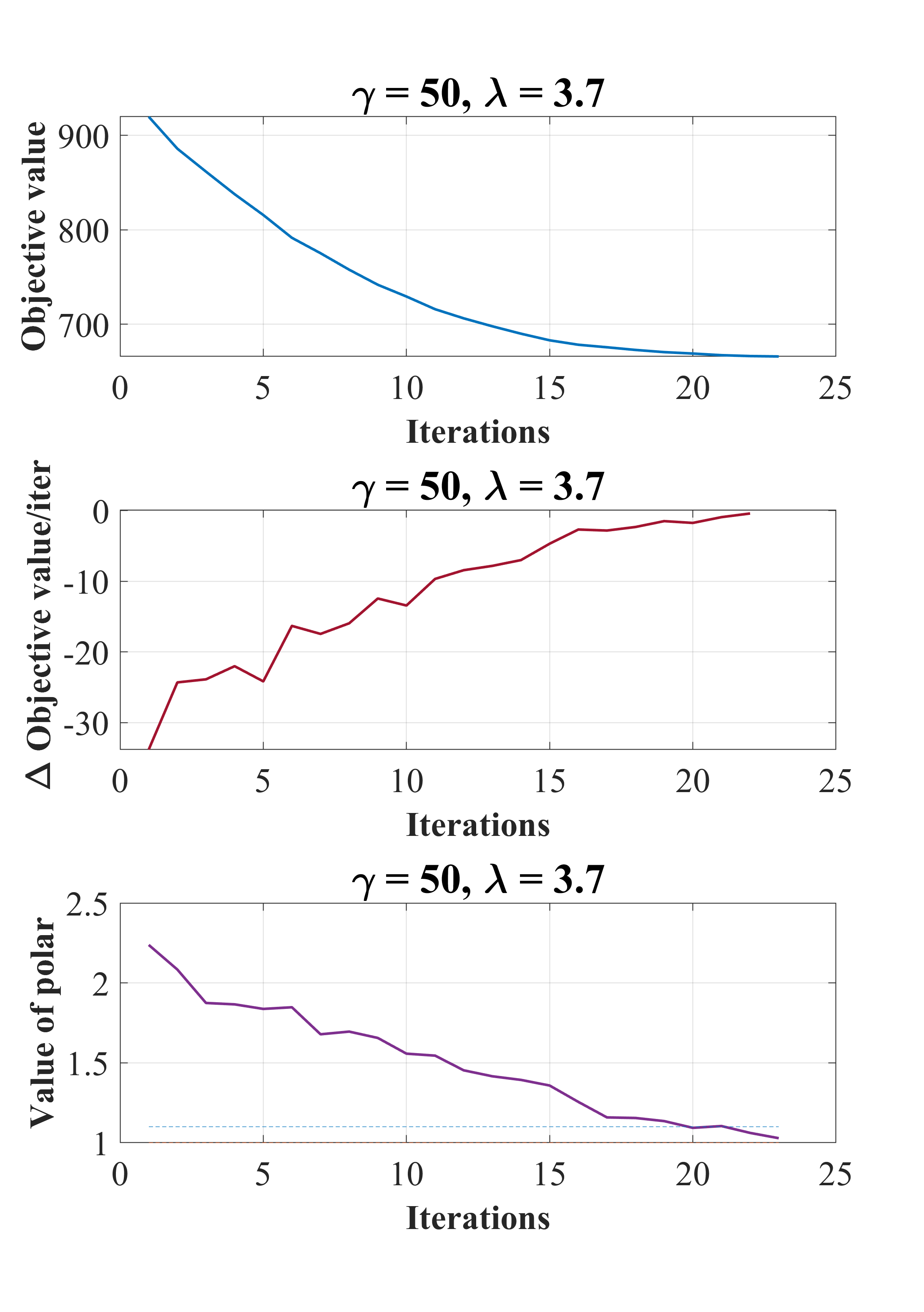}
    \caption{Objective value, rate of change of objective value per iteration and the value of polar after each iteration of the meta-algorithm for the simulation in \textbf{Characterizing Multi-Segment Transmission Lines} \S \ref{sec:results} with values of regularization parameters $\gamma = 50$, $\lambda = 3.7$}
    \label{fig:workflow_b}
\end{figure}
\begin{figure}
    \centering
   \includegraphics[width=0.95\textwidth]{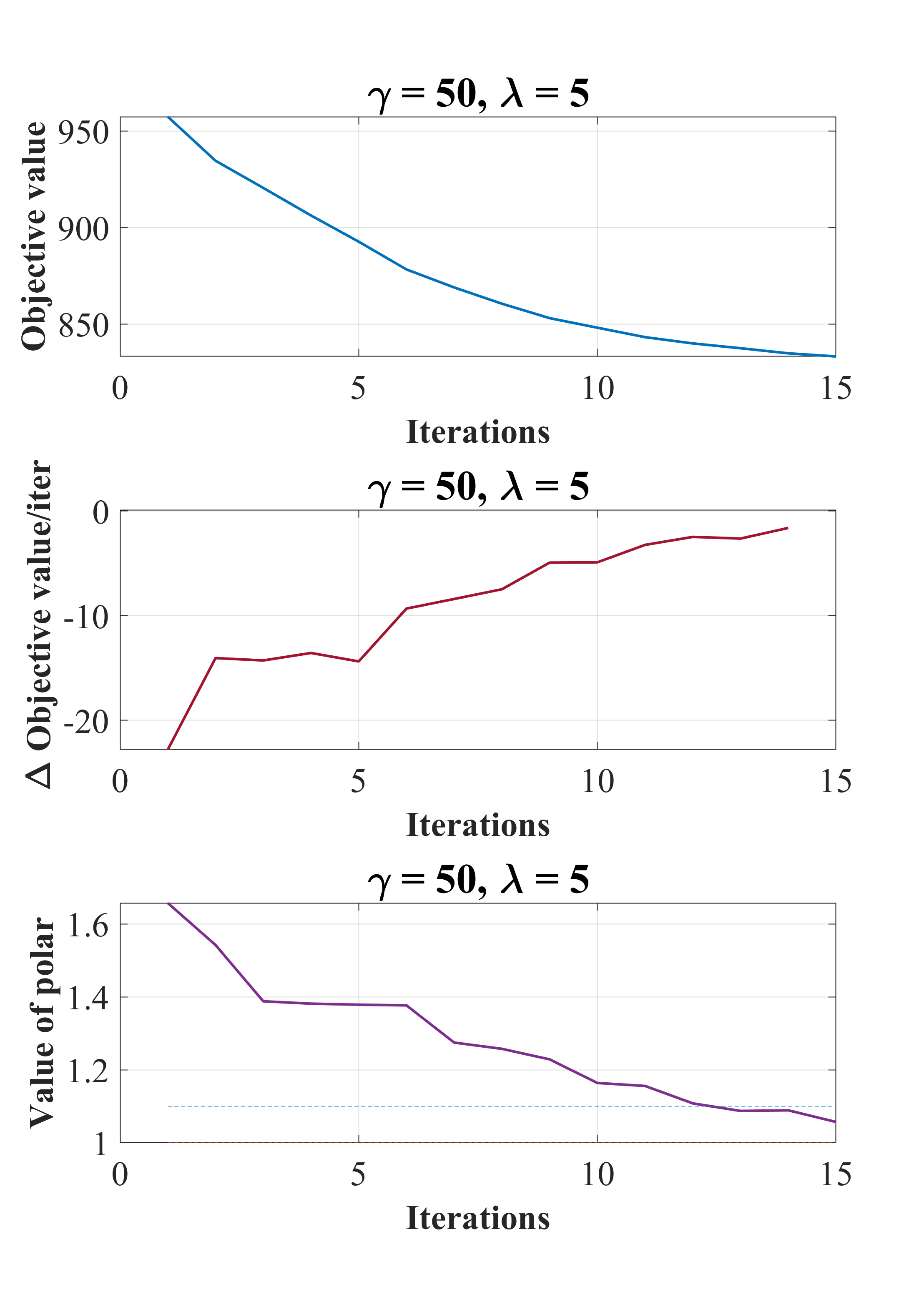}
    \caption{Objective value, rate of change of objective value per iteration and the value of polar after each iteration of the meta-algorithm for the simulation in \textbf{Characterizing Multi-Segment Transmission Lines} \S \ref{sec:results} with values of regularization parameters $\gamma = 50$, $\lambda = 5$}
    \label{fig:workflow_c}
\end{figure}


In addition to these optimization results, we also show additional examples of columns of $\D$ for different values of $\gamma$ ($500$ and $5000$) with $\lambda$ fixed at 0.6 to show the variation of the columns of $\D$ with $\gamma$.

In Figures~\ref{dictatoms_multi_a} and ~\ref{dictatoms_multi_b}, we observe that with increasing $\gamma$, the performance of the algorithm in distinguishing the two regions reduces: energy in the dictionary column is no longer confined to the segment containing the same material, it is distributed over the entire column. This can also be observed quantitatively from comparing the partitioned normalized energies (energy on each of the two segments, partitioned by the violet line, of the normalized columns of $\D$) in Figures~\ref{energies_a},~\ref{energies_b} and Figures~\ref{dictatoms_multi_a},\ref{dictatoms_multi_b}.

\begin{figure}[ht!]
    \centering
    \includegraphics[width=0.65\textwidth]{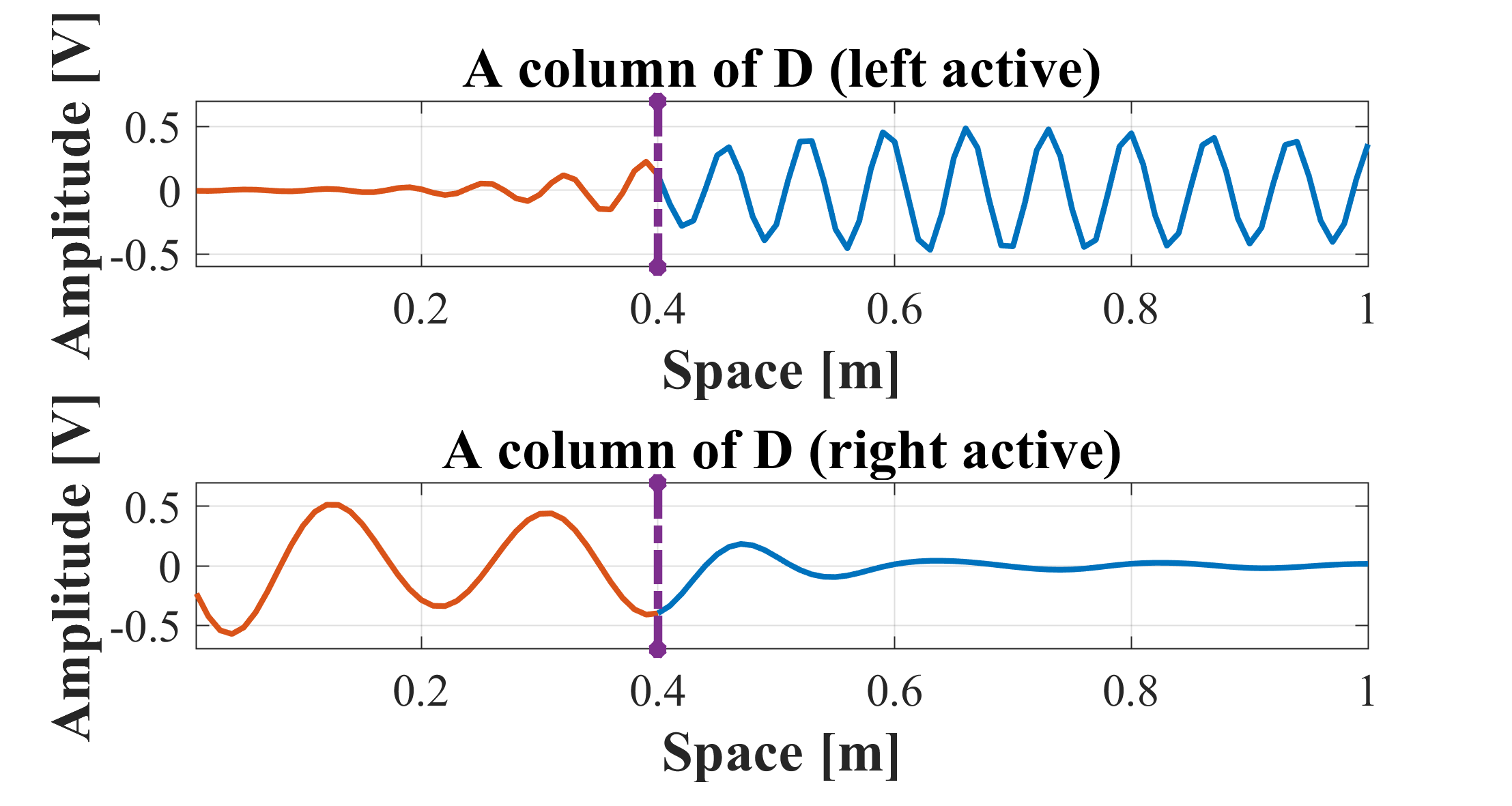}
    \caption{Two Columns of $\D$ obtained from wave-informed matrix factorization when $\gamma = 500$ $\lambda = 0.6$.}
    \label{dictatoms_multi_a}
\end{figure}

\begin{figure}[ht!]
    \centering
    \includegraphics[width=0.65\textwidth]{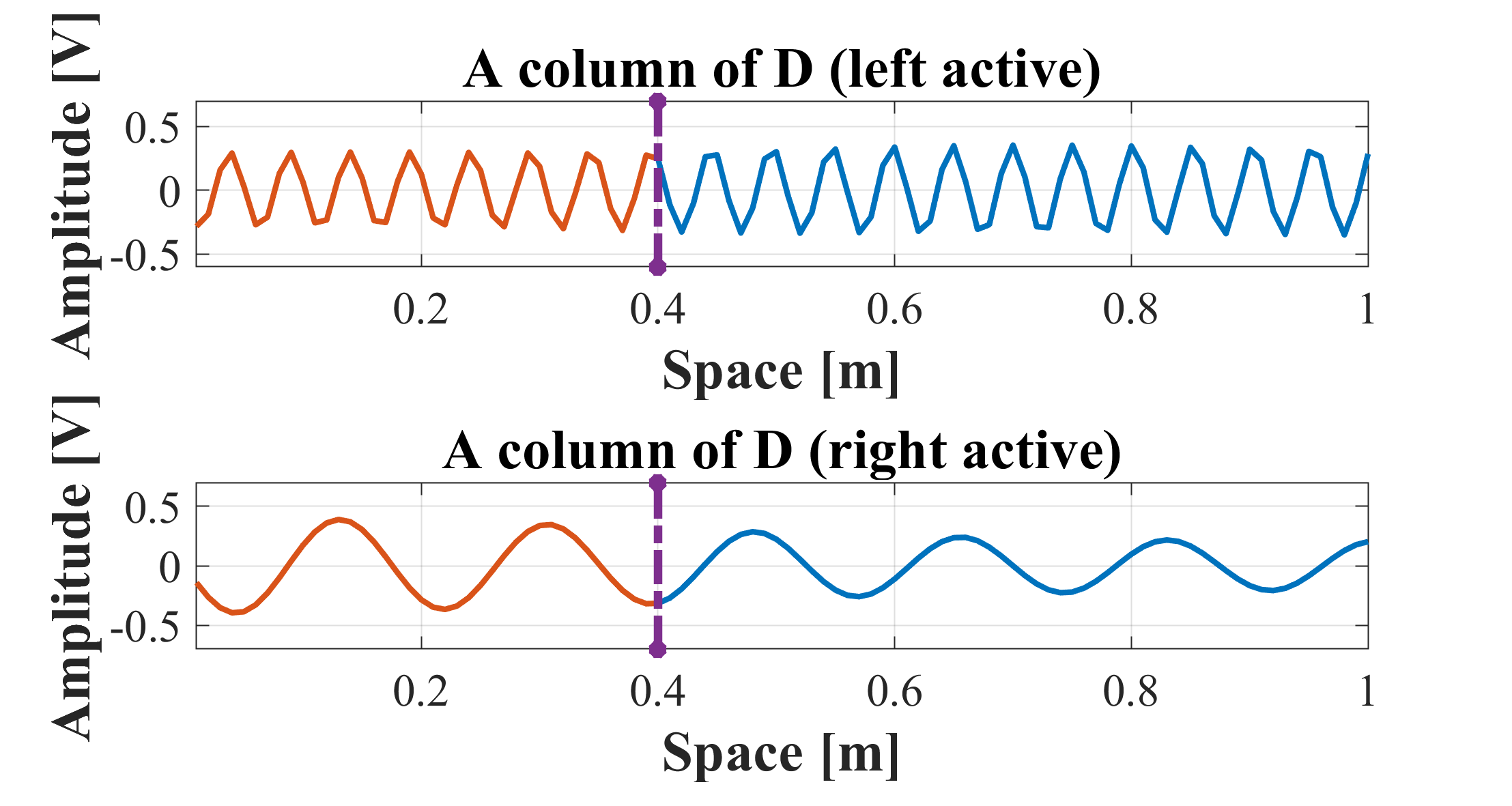}
    \caption{Two Columns of $\D$ obtained from wave-informed matrix factorization when $\gamma = 5000$ $\lambda = 0.6$.}
    \label{dictatoms_multi_b}
\end{figure}

\begin{figure}[ht!]
    \centering
   \includegraphics[width=0.65\textwidth]{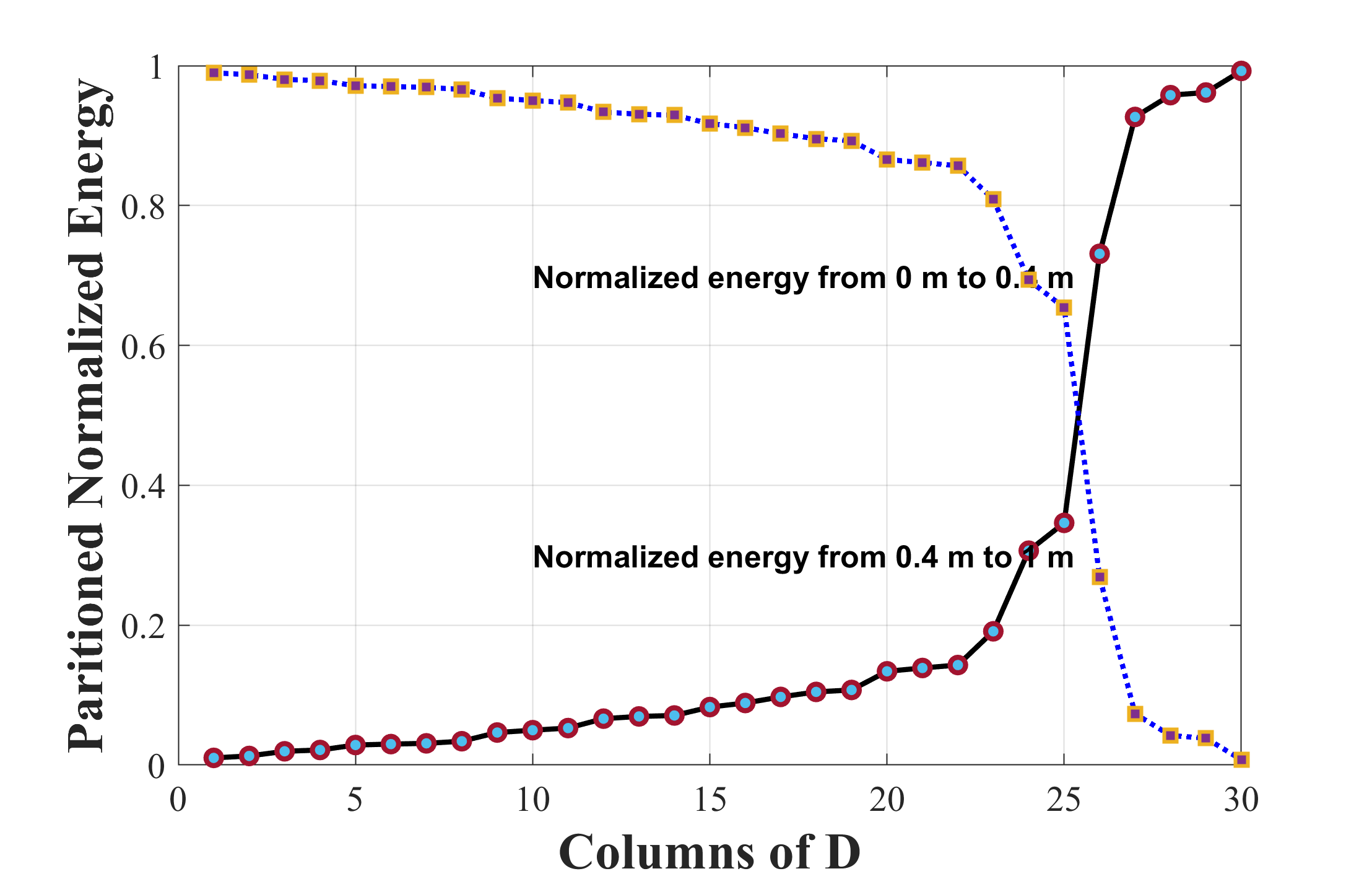}
    \caption{Two Columns of $\D$ obtained from wave-informed matrix factorization when $\gamma = 500$ $\lambda = 0.6$.}
    \label{energies_a}
\end{figure}

\begin{figure}[ht!]
    \centering
    \includegraphics[width=0.65\textwidth]{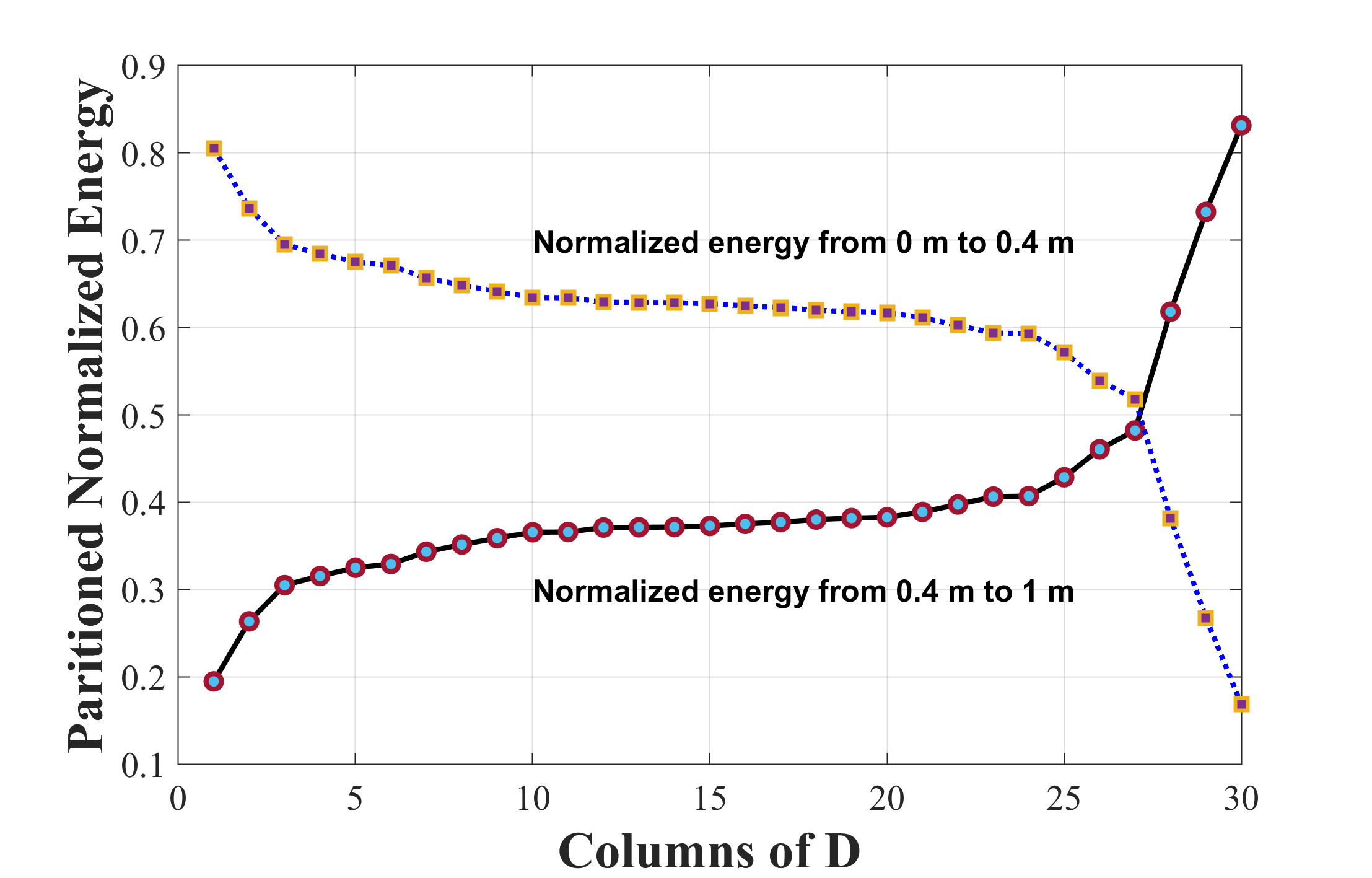}
    \caption{Two Columns of $\D$ obtained from wave-informed matrix factorization when $\gamma = 5000$ $\lambda = 0.6$.}
    \label{energies_b}
\end{figure}

\subsection{A Vibrating String}
For the experiment described in subsection \textbf{A Fixed Vibrating String} of \S \ref{sec:results}, we show the squared difference between the actual matrix $\D$ and the one estimated from the algorithm in Table \ref{table:wave_1}, Table \ref{table:wave_2}, and Table \ref{table:wave_3} (for different scenarios). We can calculate the percentage error as
\begin{gather}
 \% \textnormal{error} = \frac{\| \hat{\D} - \D \|_F}{\|\D\|_F} \times 100
\end{gather}
where $\hat{\D}$ is the matrix estimated from the algorithm after permuting the columns to optimally align with the ground-truth $\D$, zero-padding $\D$ or $\hat{\D}$ (whichever has fewer columns) so that the two matrices have the same number of columns, and normalizing the non-zero columns of $\hat{\D}$ to have unit $\ell_2$ norm.

Since the matrix is normalized and the actual $\D$ matrix contains 10 columns (corresponding to 10 modes in the data simulated), we know that $\| \D \|_F = 10$. We define error  $\| \hat{\D} - \D \|^2_F$ of the order of $0.01$ ($1 \%$) to be desirable. We bold error of this order in the tables. Table~\ref{table:wave_1} describes the performance as a function of regularization as described in \eqref{eqn:wave_equ} with amplitude of order 10 and noise variance of $0.1$. We see a wide range of values of $\gamma$ and $\lambda$ satisfy our error criteria.

Table~\ref{table:wave_2} describes the performance as a function of regularization as described by the damped vibrations in \eqref{eqn:damped_wave_time} with amplitude of order 10 and noise variance of $0.1$. We see that the error values satisfy our chosen criteria beyond a particular value of regularization $\gamma$ ($10^5$) and are also limited in the range of $\lambda$. The range of $\lambda$ that satisfies the chosen error condition also reduces with increasing values of $\gamma$.

Table~\ref{table:wave_3} describes the performance as a function of regularization as described in \eqref{eqn:wave_equ} with amplitude of order 10 and noise variance of 10. We observe that a narrow range of regularization values satisfy the error criteria. We thus clearly see the effect of noise in this case. Based on our observations, we believe the large regularization values over fit to the noise with high frequency components, resulting in poor performance. We believe low regularization values do not enforce enough of the wave-constraint, resulting in noisy components.

In Figures \ref{fig:accommodations_a}, \ref{fig:accommodations_b}, we show columns of $\D$ obtained from low-rank factorization and wave-informed matrix factorization with a signal amplitude of order $10$ and noise variance of $10$. For the low-rank factorization, we choose $\lambda = 2592$ and for wave-informed matrix factorization we choose $\gamma = 10^7$ and $\lambda = 2200$ ($\gamma$ and $\lambda$ are chosen corresponding to the minimum value in Table~\ref{table:wave_3}). The columns of $\D$ obtained from low-rank matrix factorization in Figure~\ref{fig:accommodations_a} are noisy and not always decipherable as sinusoids. In contrast, columns of $\D$ obtained from wave-informed matrix factorization in Figure \ref{fig:accommodations_b} on the other hand are completely noiseless and have a clearly measurable wavenumber.

For the case of \eqref{eqn:damped_wave_space}, when $R=6$, we show an approximate recovery of all the modes in Figure~\ref{fig:damped_sines_all}.

\setlength{\arrayrulewidth}{0.65mm}
\setlength{\tabcolsep}{4pt}
\begin{table}[H]

\captionof{table}{Squared Frobenius norm of the difference between actual $\D$ and the one obtained from Wave-informed matrix factorization with noise of variance 0.1 and data of the form \eqref{eqn:wave_equ}}
\begin{tabular}{@{}|r|r|r|r|r|r|r|r|r|r|r|r|r|@{}}
\hline
\multicolumn{1}{|c|}{} & \multicolumn{12}{|c|}{Regularization $\left(\gamma\right)$} \\
\hline
\multicolumn{1}{|l|}{$\lambda$} & \multicolumn{1}{l|}{$10^1$} & \multicolumn{1}{l|}{$10^{2}$} & \multicolumn{1}{l|}{$10^3$} & \multicolumn{1}{l|}{$10^4$} & \multicolumn{1}{l|}{$10^5$} & \multicolumn{1}{l|}{$10^6$} & \multicolumn{1}{l|}{$10^7$} & \multicolumn{1}{l|}{$10^8$} & \multicolumn{1}{l|}{$10^9$} & \multicolumn{1}{l|}{$10^{10}$} & \multicolumn{1}{l|}{$10^{11}$} & \multicolumn{1}{l|}{$10^{12}$} \\ \hline
200                          & \cellcolor[HTML]{FFD666}\textbf{0.0040}    & \cellcolor[HTML]{FFD666}\textbf{0.0040}    & \cellcolor[HTML]{FFD666}\textbf{0.0041}    & \cellcolor[HTML]{FFD666}\textbf{0.0041}    & \cellcolor[HTML]{FFD666}\textbf{0.0041}    & \cellcolor[HTML]{FFD666}\textbf{0.0038}    & \cellcolor[HTML]{FFD666}\textbf{0.0035}    & \cellcolor[HTML]{FFDC7F}1.0023             & \cellcolor[HTML]{FFDC7F}1.0013             & \cellcolor[HTML]{FFDC7F}1.0013                  & \cellcolor[HTML]{FFD666}\textbf{0.0013}         & \cellcolor[HTML]{FFD666}\textbf{0.0013}         \\ \hline
300                          & \cellcolor[HTML]{FFD666}\textbf{0.0040}    & \cellcolor[HTML]{FFD666}\textbf{0.0040}    & \cellcolor[HTML]{FFD666}\textbf{0.0041}    & \cellcolor[HTML]{FFD666}\textbf{0.0041}    & \cellcolor[HTML]{FFD666}\textbf{0.0041}    & \cellcolor[HTML]{FFD666}\textbf{0.0038}    & \cellcolor[HTML]{FFD666}\textbf{0.0035}    & \cellcolor[HTML]{FFDC7F}1.0023             & \cellcolor[HTML]{FFDC7F}1.0013             & \cellcolor[HTML]{FFD666}\textbf{0.0013}         & \cellcolor[HTML]{FFD666}\textbf{0.0013}         & \cellcolor[HTML]{FFD666}\textbf{0.0013}         \\ \hline
400                          & \cellcolor[HTML]{FFD666}\textbf{0.0040}    & \cellcolor[HTML]{FFD666}\textbf{0.0040}    & \cellcolor[HTML]{FFD666}\textbf{0.0041}    & \cellcolor[HTML]{FFD666}\textbf{0.0041}    & \cellcolor[HTML]{FFD666}\textbf{0.0041}    & \cellcolor[HTML]{FFD666}\textbf{0.0038}    & \cellcolor[HTML]{FFD666}\textbf{0.0035}    & \cellcolor[HTML]{FFEAB4}3.0598             & \cellcolor[HTML]{FFD666}\textbf{0.0013}    & \cellcolor[HTML]{FFD666}\textbf{0.0013}         & \cellcolor[HTML]{FFD666}\textbf{0.0013}         & \cellcolor[HTML]{FFD666}\textbf{0.0013}         \\ \hline
500                          & \cellcolor[HTML]{FFD666}\textbf{0.0040}    & \cellcolor[HTML]{FFD666}\textbf{0.0040}    & \cellcolor[HTML]{FFD666}\textbf{0.0041}    & \cellcolor[HTML]{FFD666}\textbf{0.0041}    & \cellcolor[HTML]{FFD666}\textbf{0.0041}    & \cellcolor[HTML]{FFD666}\textbf{0.0038}    & \cellcolor[HTML]{FFDC7F}1.0035             & \cellcolor[HTML]{FFEAB3}3.0438             & \cellcolor[HTML]{FFD666}\textbf{0.0013}    & \cellcolor[HTML]{FFD666}\textbf{0.0013}         & \cellcolor[HTML]{FFD666}\textbf{0.0013}         & \cellcolor[HTML]{FFD666}\textbf{0.0013}         \\ \hline
600                          & \cellcolor[HTML]{FFD666}\textbf{0.0040}    & \cellcolor[HTML]{FFD666}\textbf{0.0040}    & \cellcolor[HTML]{FFD666}\textbf{0.0041}    & \cellcolor[HTML]{FFD666}\textbf{0.0041}    & \cellcolor[HTML]{FFD666}\textbf{0.0041}    & \cellcolor[HTML]{FFD666}\textbf{0.0038}    & \cellcolor[HTML]{FFD666}\textbf{0.0035}    & \cellcolor[HTML]{FFEAB3}3.0342             & \cellcolor[HTML]{FFD666}\textbf{0.0013}    & \cellcolor[HTML]{FFD666}\textbf{0.0013}         & \cellcolor[HTML]{FFD666}\textbf{0.0013}         & \cellcolor[HTML]{FFD666}\textbf{0.0013}         \\ \hline
700                          & \cellcolor[HTML]{FFD666}\textbf{0.0040}    & \cellcolor[HTML]{FFD666}\textbf{0.0040}    & \cellcolor[HTML]{FFD666}\textbf{0.0041}    & \cellcolor[HTML]{FFD666}\textbf{0.0041}    & \cellcolor[HTML]{FFD666}\textbf{0.0041}    & \cellcolor[HTML]{FFD666}\textbf{0.0038}    & \cellcolor[HTML]{FFD666}\textbf{0.0035}    & \cellcolor[HTML]{FFE399}2.0015             & \cellcolor[HTML]{FFD666}\textbf{0.0013}    & \cellcolor[HTML]{FFD666}\textbf{0.0013}         & \cellcolor[HTML]{FFD666}\textbf{0.0013}         & \cellcolor[HTML]{FFDC7F}1.0011                  \\ \hline
800                          & \cellcolor[HTML]{FFD666}\textbf{0.0040}    & \cellcolor[HTML]{FFD666}\textbf{0.0040}    & \cellcolor[HTML]{FFD666}\textbf{0.0041}    & \cellcolor[HTML]{FFD666}\textbf{0.0041}    & \cellcolor[HTML]{FFD666}\textbf{0.0041}    & \cellcolor[HTML]{FFD666}\textbf{0.0038}    & \cellcolor[HTML]{FFD666}\textbf{0.0035}    & \cellcolor[HTML]{FFDC7F}1.0014             & \cellcolor[HTML]{FFD666}\textbf{0.0013}    & \cellcolor[HTML]{FFD666}\textbf{0.0013}         & \cellcolor[HTML]{FFD666}\textbf{0.0013}         & \cellcolor[HTML]{FFDC7F}1.0011                  \\ \hline
900                          & \cellcolor[HTML]{FFD666}\textbf{0.0040}    & \cellcolor[HTML]{FFD666}\textbf{0.0040}    & \cellcolor[HTML]{FFD666}\textbf{0.0041}    & \cellcolor[HTML]{FFD666}\textbf{0.0041}    & \cellcolor[HTML]{FFD666}\textbf{0.0041}    & \cellcolor[HTML]{FFD666}\textbf{0.0038}    & \cellcolor[HTML]{FFDC7F}1.0035             & \cellcolor[HTML]{FFDC7F}1.0013             & \cellcolor[HTML]{FFD666}\textbf{0.0013}    & \cellcolor[HTML]{FFD666}\textbf{0.0013}         & \cellcolor[HTML]{FFD666}\textbf{0.0013}         & \cellcolor[HTML]{FFDC7F}1.0011                  \\ \hline
1000                         & \cellcolor[HTML]{FFD666}\textbf{0.0040}    & \cellcolor[HTML]{FFD666}\textbf{0.0040}    & \cellcolor[HTML]{FFD666}\textbf{0.0041}    & \cellcolor[HTML]{FFD666}\textbf{0.0041}    & \cellcolor[HTML]{FFD666}\textbf{0.0041}    & \cellcolor[HTML]{FFD666}\textbf{0.0038}    & \cellcolor[HTML]{FFDC7F}1.0035             & \cellcolor[HTML]{FFDC7F}1.0013             & \cellcolor[HTML]{FFD666}\textbf{0.0013}    & \cellcolor[HTML]{FFD666}\textbf{0.0013}         & \cellcolor[HTML]{FFD666}\textbf{0.0013}         & \cellcolor[HTML]{FFE398}2.0009                  \\ \hline
1100                         & \cellcolor[HTML]{FFD666}\textbf{0.0040}    & \cellcolor[HTML]{FFD666}\textbf{0.0040}    & \cellcolor[HTML]{FFD666}\textbf{0.0041}    & \cellcolor[HTML]{FFD666}\textbf{0.0041}    & \cellcolor[HTML]{FFD666}\textbf{0.0041}    & \cellcolor[HTML]{FFD666}\textbf{0.0038}    & \cellcolor[HTML]{FFDC7F}1.0035             & \cellcolor[HTML]{FFDC7F}1.0014             & \cellcolor[HTML]{FFD666}\textbf{0.0013}    & \cellcolor[HTML]{FFD666}\textbf{0.0013}         & \cellcolor[HTML]{FFD666}\textbf{0.0013}         & \cellcolor[HTML]{FFEAB2}3.0009                  \\ \hline
1400                         & \cellcolor[HTML]{FFD666}\textbf{0.0040}    & \cellcolor[HTML]{FFD666}\textbf{0.0040}    & \cellcolor[HTML]{FFD666}\textbf{0.0041}    & \cellcolor[HTML]{FFD666}\textbf{0.0041}    & \cellcolor[HTML]{FFD666}\textbf{0.0041}    & \cellcolor[HTML]{FFD666}\textbf{0.0038}    & \cellcolor[HTML]{FFDC7F}1.0035             & \cellcolor[HTML]{FFEAB2}2.9963             & \cellcolor[HTML]{FFD666}\textbf{0.0013}    & \cellcolor[HTML]{FFD666}\textbf{0.0013}         & \cellcolor[HTML]{FFD666}\textbf{0.0013}         & \cellcolor[HTML]{FFF8E5}5.0005                  \\ \hline
1800                         & \cellcolor[HTML]{FFD666}\textbf{0.0040}    & \cellcolor[HTML]{FFD666}\textbf{0.0040}    & \cellcolor[HTML]{FFD666}\textbf{0.0041}    & \cellcolor[HTML]{FFD666}\textbf{0.0041}    & \cellcolor[HTML]{FFD666}\textbf{0.0041}    & \cellcolor[HTML]{FFD666}\textbf{0.0038}    & \cellcolor[HTML]{FFDC7F}1.0035             & \cellcolor[HTML]{FFDC7F}1.0019             & \cellcolor[HTML]{FFD666}\textbf{0.0013}    & \cellcolor[HTML]{FFD666}\textbf{0.0013}         & \cellcolor[HTML]{FFD666}\textbf{0.0013}         & \cellcolor[HTML]{FFFEFE}6.0004                  \\ \hline
2000                         & \cellcolor[HTML]{FFD666}\textbf{0.0040}    & \cellcolor[HTML]{FFD666}\textbf{0.0040}    & \cellcolor[HTML]{FFD666}\textbf{0.0041}    & \cellcolor[HTML]{FFD666}\textbf{0.0041}    & \cellcolor[HTML]{FFD666}\textbf{0.0041}    & \cellcolor[HTML]{FFD666}\textbf{0.0038}    & \cellcolor[HTML]{FFDC7F}1.0035             & \cellcolor[HTML]{FFDC7F}1.0018             & \cellcolor[HTML]{FFD666}\textbf{0.0013}    & \cellcolor[HTML]{FFD666}\textbf{0.0013}         & \cellcolor[HTML]{FFDC7F}1.0011                  & \cellcolor[HTML]{FFFEFE}6.0004                  \\ \hline
2100                         & \cellcolor[HTML]{FFD666}\textbf{0.0040}    & \cellcolor[HTML]{FFD666}\textbf{0.0040}    & \cellcolor[HTML]{FFD666}\textbf{0.0041}    & \cellcolor[HTML]{FFD666}\textbf{0.0041}    & \cellcolor[HTML]{FFD666}\textbf{0.0041}    & \cellcolor[HTML]{FFD666}\textbf{0.0038}    & \cellcolor[HTML]{FFDC7F}\textbf{1.0035}    & \cellcolor[HTML]{FFDC7F}1.0017             & \cellcolor[HTML]{FFD666}\textbf{0.0013}    & \cellcolor[HTML]{FFD666}\textbf{0.0013}         & \cellcolor[HTML]{FFDC7F}1.0011                  & \cellcolor[HTML]{FFFEFE}6.0004                     \\ \hline
2200                         & \cellcolor[HTML]{FFD666}\textbf{0.0040}    & \cellcolor[HTML]{FFD666}\textbf{0.0040}    & \cellcolor[HTML]{FFD666}\textbf{0.0041}    & \cellcolor[HTML]{FFD666}\textbf{0.0041}    & \cellcolor[HTML]{FFD666}\textbf{0.0041}    & \cellcolor[HTML]{FFD666}\textbf{0.0038}    & \cellcolor[HTML]{FFDC7F}\textbf{1.0035}    & \cellcolor[HTML]{FFDC7F}1.0017             & \cellcolor[HTML]{FFD666}\textbf{0.0013}    & \cellcolor[HTML]{FFD666}\textbf{0.0013}         & \cellcolor[HTML]{FFDC7F}1.0011                  & \cellcolor[HTML]{FFFEFE}6.0004                  \\ \hline
\end{tabular}
\label{table:wave_1}
\end{table}

\setlength{\arrayrulewidth}{0.65mm}
\setlength{\tabcolsep}{4pt}
\begin{table}[H]
\captionof{table}{Squared Frobenius norm of the difference between actual $\D$ and the one obtained from Wave-informed matrix factorization with noise of variance 0.1 and data of the form \eqref{eqn:damped_wave_time}}
\begin{tabular}{@{}|r|r|r|r|r|r|r|r|r|r|r|r|r|@{}}
\hline
\multicolumn{1}{|c|}{} & \multicolumn{12}{|c|}{Regularization $\left(\gamma\right)$} \\
\hline
\multicolumn{1}{|l|}{$\lambda$} & \multicolumn{1}{l|}{$10^1$} & \multicolumn{1}{l|}{$10^{2}$} & \multicolumn{1}{l|}{$10^3$} & \multicolumn{1}{l|}{$10^4$} & \multicolumn{1}{l|}{$10^5$} & \multicolumn{1}{l|}{$10^6$} & \multicolumn{1}{l|}{$10^7$} & \multicolumn{1}{l|}{$10^8$} & \multicolumn{1}{l|}{$10^9$} & \multicolumn{1}{l|}{$10^{10}$} & \multicolumn{1}{l|}{$10^{11}$} & \multicolumn{1}{l|}{$10^{12}$} \\ \hline
200                          & \cellcolor[HTML]{FFD76C}0.2846             & \cellcolor[HTML]{FFD76C}0.2845             & \cellcolor[HTML]{FFD76C}0.2829             & \cellcolor[HTML]{FFD76B}0.2655             & \cellcolor[HTML]{FFD669}0.1606             & \cellcolor[HTML]{FFD667}\textbf{0.0521}    & \cellcolor[HTML]{FFD666}\textbf{0.0165}    & \cellcolor[HTML]{FFD666}\textbf{0.0018}    & \cellcolor[HTML]{FFD666}\textbf{0.0013}    & \cellcolor[HTML]{FFD666}\textbf{0.0013}         & \cellcolor[HTML]{FFDB7B}1.0011                  & \cellcolor[HTML]{FFE7A7}3.0006                  \\ \hline
300                          & \cellcolor[HTML]{FFD76C}0.2846             & \cellcolor[HTML]{FFD76C}0.2845             & \cellcolor[HTML]{FFD76C}0.2829             & \cellcolor[HTML]{FFD76B}0.2655             & \cellcolor[HTML]{FFD669}0.1606             & \cellcolor[HTML]{FFD667}\textbf{0.0521}    & \cellcolor[HTML]{FFD666}\textbf{0.0165}    & \cellcolor[HTML]{FFD666}\textbf{0.0018}    & \cellcolor[HTML]{FFD666}\textbf{0.0013}    & \cellcolor[HTML]{FFD666}\textbf{0.0013}         & \cellcolor[HTML]{FFDB7B}1.0011                  & \cellcolor[HTML]{FFEDBD}4.0005                  \\ \hline
400                          & \cellcolor[HTML]{FFD76C}0.2846             & \cellcolor[HTML]{FFD76C}0.2845             & \cellcolor[HTML]{FFD76C}0.2829             & \cellcolor[HTML]{FFD76B}0.2655             & \cellcolor[HTML]{FFD669}0.1606             & \cellcolor[HTML]{FFD667}\textbf{0.0521}    & \cellcolor[HTML]{FFD666}\textbf{0.0165}    & \cellcolor[HTML]{FFD666}\textbf{0.0018}    & \cellcolor[HTML]{FFD666}\textbf{0.0013}    & \cellcolor[HTML]{FFD666}\textbf{0.0013}         & \cellcolor[HTML]{FFDB7B}1.0011                  & \cellcolor[HTML]{FFF3D3}5.0005                  \\ \hline
500                          & \cellcolor[HTML]{FFD76C}0.2846             & \cellcolor[HTML]{FFD76C}0.2845             & \cellcolor[HTML]{FFD76C}0.2829             & \cellcolor[HTML]{FFD76B}0.2655             & \cellcolor[HTML]{FFD669}0.1606             & \cellcolor[HTML]{FFD667}\textbf{0.0521}    & \cellcolor[HTML]{FFD666}\textbf{0.0165}    & \cellcolor[HTML]{FFD666}\textbf{0.0018}    & \cellcolor[HTML]{FFD666}\textbf{0.0013}    & \cellcolor[HTML]{FFD666}\textbf{0.0013}         & \cellcolor[HTML]{FFDB7B}1.0011                  & \cellcolor[HTML]{FFF9E9}6.0004                  \\ \hline
600                          & \cellcolor[HTML]{FFD76C}0.2846             & \cellcolor[HTML]{FFD76C}0.2845             & \cellcolor[HTML]{FFD76C}0.2829             & \cellcolor[HTML]{FFD76B}0.2655             & \cellcolor[HTML]{FFD669}0.1606             & \cellcolor[HTML]{FFD667}\textbf{0.0521}    & \cellcolor[HTML]{FFD666}\textbf{0.0165}    & \cellcolor[HTML]{FFD666}\textbf{0.0018}    & \cellcolor[HTML]{FFD666}\textbf{0.0013}    & \cellcolor[HTML]{FFD666}\textbf{0.0013}         & \cellcolor[HTML]{FFE7A7}3.0006                  & \cellcolor[HTML]{FFF9E9}6.0004                  \\ \hline
700                          & \cellcolor[HTML]{FFD76C}0.2846             & \cellcolor[HTML]{FFD76C}0.2845             & \cellcolor[HTML]{FFD76C}0.2829             & \cellcolor[HTML]{FFD76B}0.2655             & \cellcolor[HTML]{FFD669}0.1606             & \cellcolor[HTML]{FFD667}\textbf{0.0521}    & \cellcolor[HTML]{FFD666}\textbf{0.0165}    & \cellcolor[HTML]{FFD666}\textbf{0.0018}    & \cellcolor[HTML]{FFD666}\textbf{0.0013}    & \cellcolor[HTML]{FFD666}\textbf{0.0013}         & \cellcolor[HTML]{FFE7A7}3.0006                  & \cellcolor[HTML]{FFF9E9}6.0004                  \\ \hline
800                          & \cellcolor[HTML]{FFD76C}0.2846             & \cellcolor[HTML]{FFD76C}0.2845             & \cellcolor[HTML]{FFD76C}0.2829             & \cellcolor[HTML]{FFD76B}0.2655             & \cellcolor[HTML]{FFD669}0.1606             & \cellcolor[HTML]{FFD667}\textbf{0.0521}    & \cellcolor[HTML]{FFD666}\textbf{0.0165}    & \cellcolor[HTML]{FFD666}\textbf{0.0018}    & \cellcolor[HTML]{FFD666}\textbf{0.0013}    & \cellcolor[HTML]{FFDB7B}1.0011                  & \cellcolor[HTML]{FFEDBD}4.0005                  & \cellcolor[HTML]{FFF9E9}6.0004                  \\ \hline
900                          & \cellcolor[HTML]{FFD76C}0.2846             & \cellcolor[HTML]{FFD76C}0.2845             & \cellcolor[HTML]{FFD76C}0.2829             & \cellcolor[HTML]{FFD76B}0.2655             & \cellcolor[HTML]{FFD669}0.1606             & \cellcolor[HTML]{FFD667}\textbf{0.0521}    & \cellcolor[HTML]{FFD666}\textbf{0.0165}    & \cellcolor[HTML]{FFD666}\textbf{0.0018}    & \cellcolor[HTML]{FFD666}\textbf{0.0013}    & \cellcolor[HTML]{FFDB7B}1.0011                  & \cellcolor[HTML]{FFEDBD}4.0005                  & \cellcolor[HTML]{FFF9E9}6.0004                  \\ \hline
1000                         & \cellcolor[HTML]{FFD76C}0.2846             & \cellcolor[HTML]{FFD76C}0.2845             & \cellcolor[HTML]{FFD76C}0.2829             & \cellcolor[HTML]{FFD76B}0.2655             & \cellcolor[HTML]{FFD669}0.1606             & \cellcolor[HTML]{FFD667}\textbf{0.0521}    & \cellcolor[HTML]{FFD666}\textbf{0.0165}    & \cellcolor[HTML]{FFD666}\textbf{0.0018}    & \cellcolor[HTML]{FFD666}\textbf{0.0013}    & \cellcolor[HTML]{FFDB7B}1.0011                  & \cellcolor[HTML]{FFEDBD}4.0005                  & \cellcolor[HTML]{FFF9E9}6.0004                  \\ \hline
1100                         & \cellcolor[HTML]{FFDD82}1.2836             & \cellcolor[HTML]{FFDD82}1.2835             & \cellcolor[HTML]{FFDD81}1.2819             & \cellcolor[HTML]{FFDD81}1.2644             & \cellcolor[HTML]{FFDC7F}1.1596             & \cellcolor[HTML]{FFDC7C}1.0513             & \cellcolor[HTML]{FFDB7C}1.0162             & \cellcolor[HTML]{FFDB7B}1.0016             & \cellcolor[HTML]{FFDB7B}1.0011             & \cellcolor[HTML]{FFDB7B}1.0011                  & \cellcolor[HTML]{FFEDBD}4.0005                  & \cellcolor[HTML]{FFF9E9}6.0004                  \\ \hline
1400                         & \cellcolor[HTML]{FFE397}2.2794             & \cellcolor[HTML]{FFE397}2.2793             & \cellcolor[HTML]{FFE397}2.2777             & \cellcolor[HTML]{FFE397}2.2602             & \cellcolor[HTML]{FFE295}2.1554             & \cellcolor[HTML]{FFE192}2.0478             & \cellcolor[HTML]{FFE192}2.0152             & \cellcolor[HTML]{FFE191}2.0013             & \cellcolor[HTML]{FFE191}2.0008             & \cellcolor[HTML]{FFEDBD}4.0003                  & \cellcolor[HTML]{FFFEFE}7.0002                  & \cellcolor[HTML]{FFFEFE}7.0002                  \\ \hline
1800                         & \cellcolor[HTML]{FFEDBE}4.0460             & \cellcolor[HTML]{FFEDBE}4.0460             & \cellcolor[HTML]{FFEDBE}4.0460             & \cellcolor[HTML]{FFEDBE}4.0459             & \cellcolor[HTML]{FFEDBE}4.0452             & \cellcolor[HTML]{FFEDBE}4.0387             & \cellcolor[HTML]{FFEDBD}4.0127             & \cellcolor[HTML]{FFEDBD}4.0008             & \cellcolor[HTML]{FFEDBD}4.0003             & \cellcolor[HTML]{FFEDBD}4.0003                  & \cellcolor[HTML]{FFFEFE}7.0002                  & \cellcolor[HTML]{FFFEFE}7.0002                  \\ \hline
2000                         & \cellcolor[HTML]{FFEDBE}4.0460             & \cellcolor[HTML]{FFEDBE}4.0460             & \cellcolor[HTML]{FFEDBE}4.0460             & \cellcolor[HTML]{FFEDBE}4.0459             & \cellcolor[HTML]{FFEDBE}4.0452             & \cellcolor[HTML]{FFEDBE}4.0387             & \cellcolor[HTML]{FFEDBD}4.0127             & \cellcolor[HTML]{FFEDBD}4.0008             & \cellcolor[HTML]{FFEDBD}4.0003             & \cellcolor[HTML]{FFF3D3}5.0002                  & \cellcolor[HTML]{FFFEFE}7.0002                  & \cellcolor[HTML]{FFFEFE}7.0002                  \\ \hline
2100                         & \cellcolor[HTML]{FFEDBE}4.0460             & \cellcolor[HTML]{FFEDBE}4.0460             & \cellcolor[HTML]{FFEDBE}4.0460             & \cellcolor[HTML]{FFEDBE}4.0459             & \cellcolor[HTML]{FFEDBE}4.0452             & \cellcolor[HTML]{FFEDBE}4.0387             & \cellcolor[HTML]{FFEDBD}4.0127             & \cellcolor[HTML]{FFEDBD}4.0008             & \cellcolor[HTML]{FFEDBD}4.0003             & \cellcolor[HTML]{FFF3D3}5.0002                  & \cellcolor[HTML]{FFFEFE}7.0002                     & \cellcolor[HTML]{FFFEFE}7.0002                  \\ \hline
2200                         & \cellcolor[HTML]{FFEDBE}4.0460             & \cellcolor[HTML]{FFEDBE}4.0460             & \cellcolor[HTML]{FFEDBE}4.0460             & \cellcolor[HTML]{FFEDBE}4.0459             & \cellcolor[HTML]{FFEDBE}4.0452             & \cellcolor[HTML]{FFEDBE}4.0387             & \cellcolor[HTML]{FFEDBD}4.0127             & \cellcolor[HTML]{FFEDBD}4.0008             & \cellcolor[HTML]{FFEDBD}4.0003             & \cellcolor[HTML]{FFF3D3}5.0002                  & \cellcolor[HTML]{FFFEFE}7.0002                  & \cellcolor[HTML]{FFFEFE}7.0002                  \\ \hline
\end{tabular}
\label{table:wave_2}
\end{table}

\setlength{\arrayrulewidth}{0.65mm}
\setlength{\tabcolsep}{4pt}

\begin{table}[H]
\captionof{table}{Squared Frobenius norm of the difference between actual $\D$ and the one obtained from Wave-informed matrix factorization with noise of variance 10 and data of the form \eqref{eqn:wave_equ}}
\begin{tabular}{@{}|r|r|r|r|r|r|r|r|r|r|r|r|r|@{}}
\hline
\multicolumn{1}{|c|}{} & \multicolumn{12}{|c|}{Regularization $\left(\gamma\right)$} \\
\hline
\multicolumn{1}{|l|}{$\lambda$} & \multicolumn{1}{l|}{$10^1$} & \multicolumn{1}{l|}{$10^{2}$} & \multicolumn{1}{l|}{$10^3$} & \multicolumn{1}{l|}{$10^4$} & \multicolumn{1}{l|}{$10^5$} & \multicolumn{1}{l|}{$10^6$} & \multicolumn{1}{l|}{$10^7$} & \multicolumn{1}{l|}{$10^8$} & \multicolumn{1}{l|}{$10^9$} & \multicolumn{1}{l|}{$10^{10}$} & \multicolumn{1}{l|}{$10^{11}$} & \multicolumn{1}{l|}{$10^{12}$} \\ \hline
200                          & \cellcolor[HTML]{FFFEFE}20.220                & \cellcolor[HTML]{FFFEFE}20.219             & \cellcolor[HTML]{FFFEFE}20.219             & \cellcolor[HTML]{FFFEFE}20.216             & \cellcolor[HTML]{FFFEFE}20.202             & \cellcolor[HTML]{FFFEFE}20.141             & \cellcolor[HTML]{FFFEFD}20.063             & \cellcolor[HTML]{FFFEFD}20.022             & \cellcolor[HTML]{FFFEFD}20.001             & \cellcolor[HTML]{FFFEFD}20.001                  & \cellcolor[HTML]{FFFEFD}20.001                  & \cellcolor[HTML]{FFE6A2}7.991                   \\ \hline
300                          & \cellcolor[HTML]{FFFEFE}20.220             & \cellcolor[HTML]{FFFEFE}20.219             & \cellcolor[HTML]{FFFEFE}20.219             & \cellcolor[HTML]{FFFEFE}20.216             & \cellcolor[HTML]{FFFEFE}20.202             & \cellcolor[HTML]{FFFEFE}20.141             & \cellcolor[HTML]{FFFEFD}20.063             & \cellcolor[HTML]{FFFEFD}20.022             & \cellcolor[HTML]{FFFEFD}20.001             & \cellcolor[HTML]{FFFEFD}20.001                  & \cellcolor[HTML]{FFFEFD}20.001                  & \cellcolor[HTML]{FFE49A}6.991                   \\ \hline
400                          & \cellcolor[HTML]{FFFEFE}20.220             & \cellcolor[HTML]{FFFEFE}20.219             & \cellcolor[HTML]{FFFEFE}20.219             & \cellcolor[HTML]{FFFEFE}20.216             & \cellcolor[HTML]{FFFEFE}20.202             & \cellcolor[HTML]{FFFEFE}20.141             & \cellcolor[HTML]{FFFEFD}20.063             & \cellcolor[HTML]{FFFEFD}20.022             & \cellcolor[HTML]{FFFEFD}20.001             & \cellcolor[HTML]{FFFEFD}20.001                  & \cellcolor[HTML]{FFFEFD}20.001                  & \cellcolor[HTML]{FFE49A}6.991                   \\ \hline
500                          & \cellcolor[HTML]{FFFEFE}20.220             & \cellcolor[HTML]{FFFEFE}20.219             & \cellcolor[HTML]{FFFEFE}20.219             & \cellcolor[HTML]{FFFEFE}20.216             & \cellcolor[HTML]{FFFEFE}20.202             & \cellcolor[HTML]{FFFEFE}20.141             & \cellcolor[HTML]{FFFEFD}20.063             & \cellcolor[HTML]{FFFEFD}20.022             & \cellcolor[HTML]{FFFEFD}20.001             & \cellcolor[HTML]{FFFEFD}20.001                  & \cellcolor[HTML]{FFFEFD}20.001                  & \cellcolor[HTML]{FFE49A}6.991                   \\ \hline
600                          & \cellcolor[HTML]{FFFEFE}20.220             & \cellcolor[HTML]{FFFEFE}20.219             & \cellcolor[HTML]{FFFEFE}20.219             & \cellcolor[HTML]{FFFEFE}20.216             & \cellcolor[HTML]{FFFEFE}20.202             & \cellcolor[HTML]{FFFEFE}20.141             & \cellcolor[HTML]{FFFEFD}20.063             & \cellcolor[HTML]{FFFEFD}20.022             & \cellcolor[HTML]{FFFEFD}20.001             & \cellcolor[HTML]{FFFEFD}20.001                  & \cellcolor[HTML]{FFEAB1}10.001                  & \cellcolor[HTML]{FFE49A}6.991                   \\ \hline
700                          & \cellcolor[HTML]{FFFEFE}20.220             & \cellcolor[HTML]{FFFEFE}20.219             & \cellcolor[HTML]{FFFEFE}20.219             & \cellcolor[HTML]{FFFEFE}20.216             & \cellcolor[HTML]{FFFEFE}20.202             & \cellcolor[HTML]{FFFEFE}20.141             & \cellcolor[HTML]{FFFEFD}20.063             & \cellcolor[HTML]{FFFEFD}20.022             & \cellcolor[HTML]{FFFEFD}20.001             & \cellcolor[HTML]{FFFEFD}20.001                  & \cellcolor[HTML]{FFE293}6.001                   & \cellcolor[HTML]{FFE49A}6.991                   \\ \hline
800                          & \cellcolor[HTML]{FFFEFE}20.220             & \cellcolor[HTML]{FFFEFE}20.219             & \cellcolor[HTML]{FFFEFE}20.219             & \cellcolor[HTML]{FFFEFE}20.216             & \cellcolor[HTML]{FFFEFE}20.202             & \cellcolor[HTML]{FFFEFE}20.141             & \cellcolor[HTML]{FFFEFD}20.063             & \cellcolor[HTML]{FFFEFD}20.022             & \cellcolor[HTML]{FFFEFD}20.001             & \cellcolor[HTML]{FFFEFD}20.001                  & \cellcolor[HTML]{FFE293}6.001                   & \cellcolor[HTML]{FFE49A}6.991                   \\ \hline
900                          & \cellcolor[HTML]{FFFEFE}20.220             & \cellcolor[HTML]{FFFEFE}20.219             & \cellcolor[HTML]{FFFEFE}20.219             & \cellcolor[HTML]{FFFEFE}20.216             & \cellcolor[HTML]{FFFEFE}20.202             & \cellcolor[HTML]{FFFEFE}20.141             & \cellcolor[HTML]{FFFEFD}20.063             & \cellcolor[HTML]{FFFEFD}20.022             & \cellcolor[HTML]{FFFEFD}20.001             & \cellcolor[HTML]{FFFEFD}20.001                  & \cellcolor[HTML]{FFE293}6.001                   & \cellcolor[HTML]{FFE49A}6.991                   \\ \hline
1000                         & \cellcolor[HTML]{FFFEFE}20.220             & \cellcolor[HTML]{FFFEFE}20.219             & \cellcolor[HTML]{FFFEFE}20.219             & \cellcolor[HTML]{FFFEFE}20.216             & \cellcolor[HTML]{FFFEFE}20.202             & \cellcolor[HTML]{FFFEFE}20.141             & \cellcolor[HTML]{FFFEFD}20.063             & \cellcolor[HTML]{FFFEFD}20.022             & \cellcolor[HTML]{FFFEFD}20.001             & \cellcolor[HTML]{FFFEFD}20.001                  & \cellcolor[HTML]{FFE293}6.001                   & \cellcolor[HTML]{FFE49A}6.991                   \\ \hline
1100                         & \cellcolor[HTML]{FFFEFE}20.220             & \cellcolor[HTML]{FFFEFE}20.219             & \cellcolor[HTML]{FFFEFE}20.219             & \cellcolor[HTML]{FFFEFE}20.216             & \cellcolor[HTML]{FFFEFE}20.202             & \cellcolor[HTML]{FFFEFE}20.141             & \cellcolor[HTML]{FFFEFD}20.063             & \cellcolor[HTML]{FFFEFD}20.022             & \cellcolor[HTML]{FFFEFD}20.001             & \cellcolor[HTML]{FFFEFD}20.001                  & \cellcolor[HTML]{FFE293}6.001                   & \cellcolor[HTML]{FFDD83}3.991                   \\ \hline
1400                         & \cellcolor[HTML]{FFFEFE}20.220             & \cellcolor[HTML]{FFFEFE}20.219             & \cellcolor[HTML]{FFFEFE}20.219             & \cellcolor[HTML]{FFFEFE}20.216             & \cellcolor[HTML]{FFFEFE}20.202             & \cellcolor[HTML]{FFFEFE}20.141             & \cellcolor[HTML]{FFFEFD}20.063             & \cellcolor[HTML]{FFFEFD}20.022             & \cellcolor[HTML]{FFFEFD}20.001             & \cellcolor[HTML]{FFE8A9}9.001                   & \cellcolor[HTML]{FFE293}6.001                   & \cellcolor[HTML]{FFE292}5.978                   \\ \hline
1800                         & \cellcolor[HTML]{FFFEFE}20.220             & \cellcolor[HTML]{FFFEFE}20.218             & \cellcolor[HTML]{FFFEFE}20.219             & \cellcolor[HTML]{FFFEFE}20.216             & \cellcolor[HTML]{FFFEFE}20.202             & \cellcolor[HTML]{FFFEFE}20.141             & \cellcolor[HTML]{FFFEFD}20.063             & \cellcolor[HTML]{FFFEFD}20.022             & \cellcolor[HTML]{FFFEFD}20.001             & \cellcolor[HTML]{FFE49A}7.001                   & \cellcolor[HTML]{FFE293}6.001                   & \cellcolor[HTML]{FFE49A}6.978                   \\ \hline
2000                         & \cellcolor[HTML]{FFFEFE}20.220             & \cellcolor[HTML]{FFFEFE}20.218             & \cellcolor[HTML]{FFFEFE}20.219             & \cellcolor[HTML]{FFEAB3}10.216             & \cellcolor[HTML]{FFDC7D}3.202              & \cellcolor[HTML]{FFD666}0.141              & \cellcolor[HTML]{FFD86D}1.063              & \cellcolor[HTML]{FFDE84}4.021              & \cellcolor[HTML]{FFDE83}4.001              & \cellcolor[HTML]{FFDE83}4.001                   & \cellcolor[HTML]{FFDB7C}3.001                   & \cellcolor[HTML]{FFE293}6.000                   \\ \hline
2100                         & \cellcolor[HTML]{FFFEFE}20.220             & \cellcolor[HTML]{FFFEFE}20.218             & \cellcolor[HTML]{FFDE85}4.219              & \cellcolor[HTML]{FFD667}0.216              & \cellcolor[HTML]{FFD667}0.202              & \cellcolor[HTML]{FFD666}0.141              & \cellcolor[HTML]{FFD666}\textbf{0.063}     & \cellcolor[HTML]{FFDC7C}3.021              & \cellcolor[HTML]{FFDB7C}3.001              & \cellcolor[HTML]{FFDB7C}3.001                   & \cellcolor[HTML]{FFD974}2.001                   & \cellcolor[HTML]{FFE293}6.000                   \\ \hline
2200                         & \cellcolor[HTML]{FFFEFE}20.220             & \cellcolor[HTML]{FFDA76}2.219              & \cellcolor[HTML]{FFD667}0.219              & \cellcolor[HTML]{FFD667}0.216              & \cellcolor[HTML]{FFD667}0.202              & \cellcolor[HTML]{FFD666}0.141              & \cellcolor[HTML]{FFD666}\textbf{0.063}     & \cellcolor[HTML]{FFDC7C}3.021              & \cellcolor[HTML]{FFDB7C}3.001              & \cellcolor[HTML]{FFDB7C}3.001                   & \cellcolor[HTML]{FFD974}2.001                   & \cellcolor[HTML]{FFE293}6.000                   \\ \hline
2300                         & \cellcolor[HTML]{FFE49C}7.220              & \cellcolor[HTML]{FFD667}0.219              & \cellcolor[HTML]{FFD667}0.219              & \cellcolor[HTML]{FFD667}0.216              & \cellcolor[HTML]{FFD667}0.202              & \cellcolor[HTML]{FFD666}0.141              & \cellcolor[HTML]{FFD666}\textbf{0.063}     & \cellcolor[HTML]{FFDC7C}3.021              & \cellcolor[HTML]{FFDB7C}3.001              & \cellcolor[HTML]{FFDB7C}3.001                   & \cellcolor[HTML]{FFD974}2.001                   & \cellcolor[HTML]{FFE293}6.000                   \\ \hline
2500                         & \cellcolor[HTML]{FFD667}0.220              & \cellcolor[HTML]{FFD667}0.219              & \cellcolor[HTML]{FFD667}0.219              & \cellcolor[HTML]{FFD667}0.216              & \cellcolor[HTML]{FFD667}0.202              & \cellcolor[HTML]{FFD666}0.141              & \cellcolor[HTML]{FFD666}\textbf{0.063}     & \cellcolor[HTML]{FFDC7C}3.021              & \cellcolor[HTML]{FFDB7C}3.001              & \cellcolor[HTML]{FFDB7C}3.001                   & \cellcolor[HTML]{FFD974}2.001                   & \cellcolor[HTML]{FFE293}6.000                   \\ \hline
2700                         & \cellcolor[HTML]{FFD667}0.220              & \cellcolor[HTML]{FFD667}0.219              & \cellcolor[HTML]{FFD667}0.219              & \cellcolor[HTML]{FFD667}0.216              & \cellcolor[HTML]{FFD667}0.202              & \cellcolor[HTML]{FFD666}0.141              & \cellcolor[HTML]{FFD666}\textbf{0.063}     & \cellcolor[HTML]{FFDC7C}3.021              & \cellcolor[HTML]{FFDB7C}3.001              & \cellcolor[HTML]{FFDB7C}3.001                   & \cellcolor[HTML]{FFD974}2.001                   & \cellcolor[HTML]{FFE293}6.000                   \\ \hline
3000                         & \cellcolor[HTML]{FFD667}0.220              & \cellcolor[HTML]{FFD667}0.219              & \cellcolor[HTML]{FFD667}0.219              & \cellcolor[HTML]{FFD667}0.216              & \cellcolor[HTML]{FFD667}0.202              & \cellcolor[HTML]{FFD666}0.141              & \cellcolor[HTML]{FFD666}\textbf{0.063}     & \cellcolor[HTML]{FFDC7C}3.021              & \cellcolor[HTML]{FFDB7C}3.001              & \cellcolor[HTML]{FFDB7C}3.001                   & \cellcolor[HTML]{FFD974}2.001                   & \cellcolor[HTML]{FFE293}6.000                   \\ \hline
3200                         & \cellcolor[HTML]{FFD667}0.220              & \cellcolor[HTML]{FFD667}0.219              & \cellcolor[HTML]{FFD667}0.219              & \cellcolor[HTML]{FFD667}0.216              & \cellcolor[HTML]{FFD667}0.202              & \cellcolor[HTML]{FFD666}0.141              & \cellcolor[HTML]{FFD666}\textbf{0.063}     & \cellcolor[HTML]{FFDC7C}3.021              & \cellcolor[HTML]{FFDB7C}3.001              & \cellcolor[HTML]{FFDB7C}3.001                   & \cellcolor[HTML]{FFD974}2.001                   & \cellcolor[HTML]{FFE293}6.000                   \\ \hline
\end{tabular}
\label{table:wave_3}
\end{table}

\begin{figure}
    \centering
    \includegraphics[width=0.5\textwidth]{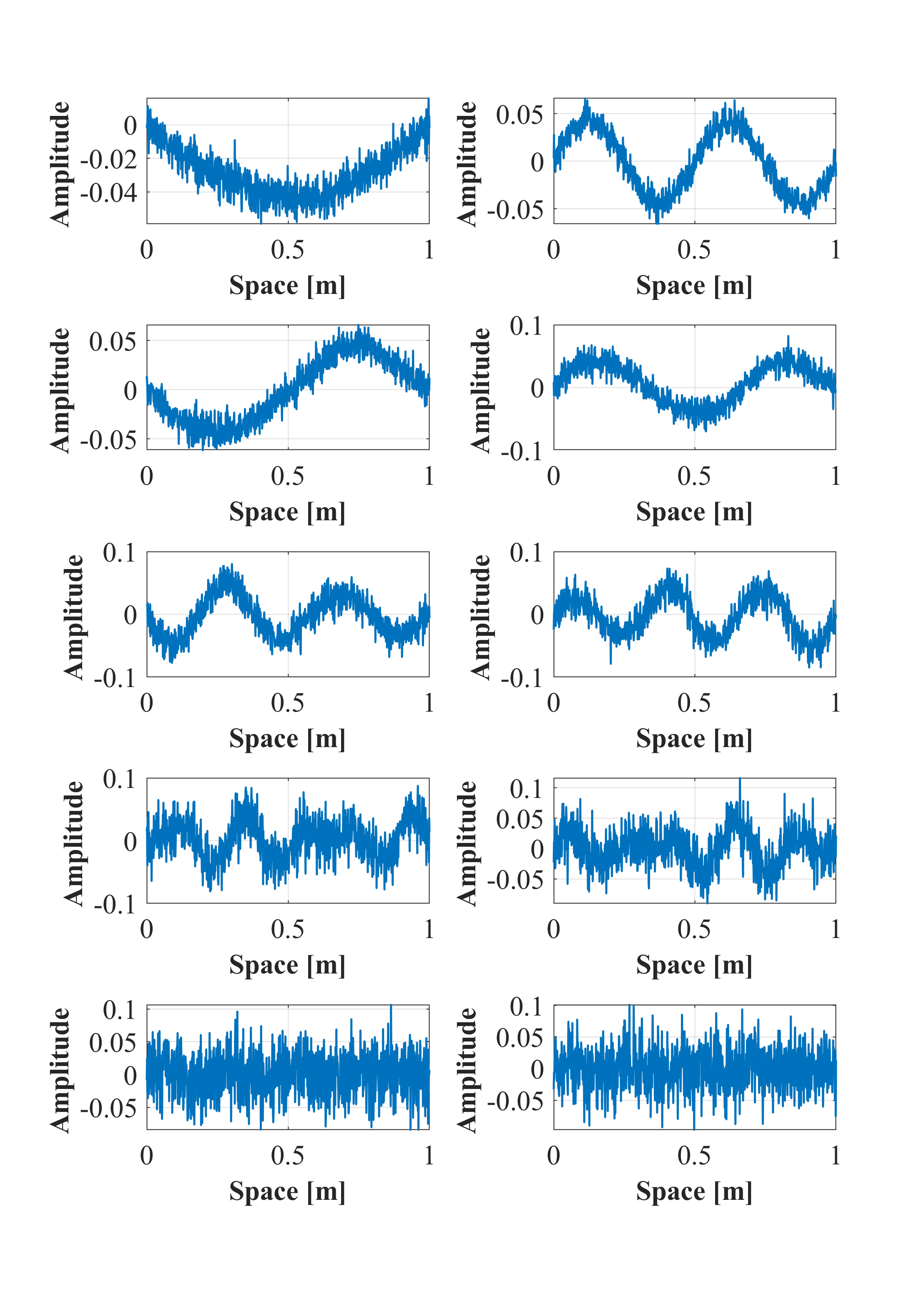}
    \caption{ \label{fig:accommodations_a}Example columns of $\D$ obtained from Low-rank matrix factorization $\lambda = 2952$. }
    \label{fig:my_label}
\end{figure}

\begin{figure}
    \centering
    \includegraphics[width=0.5\textwidth]{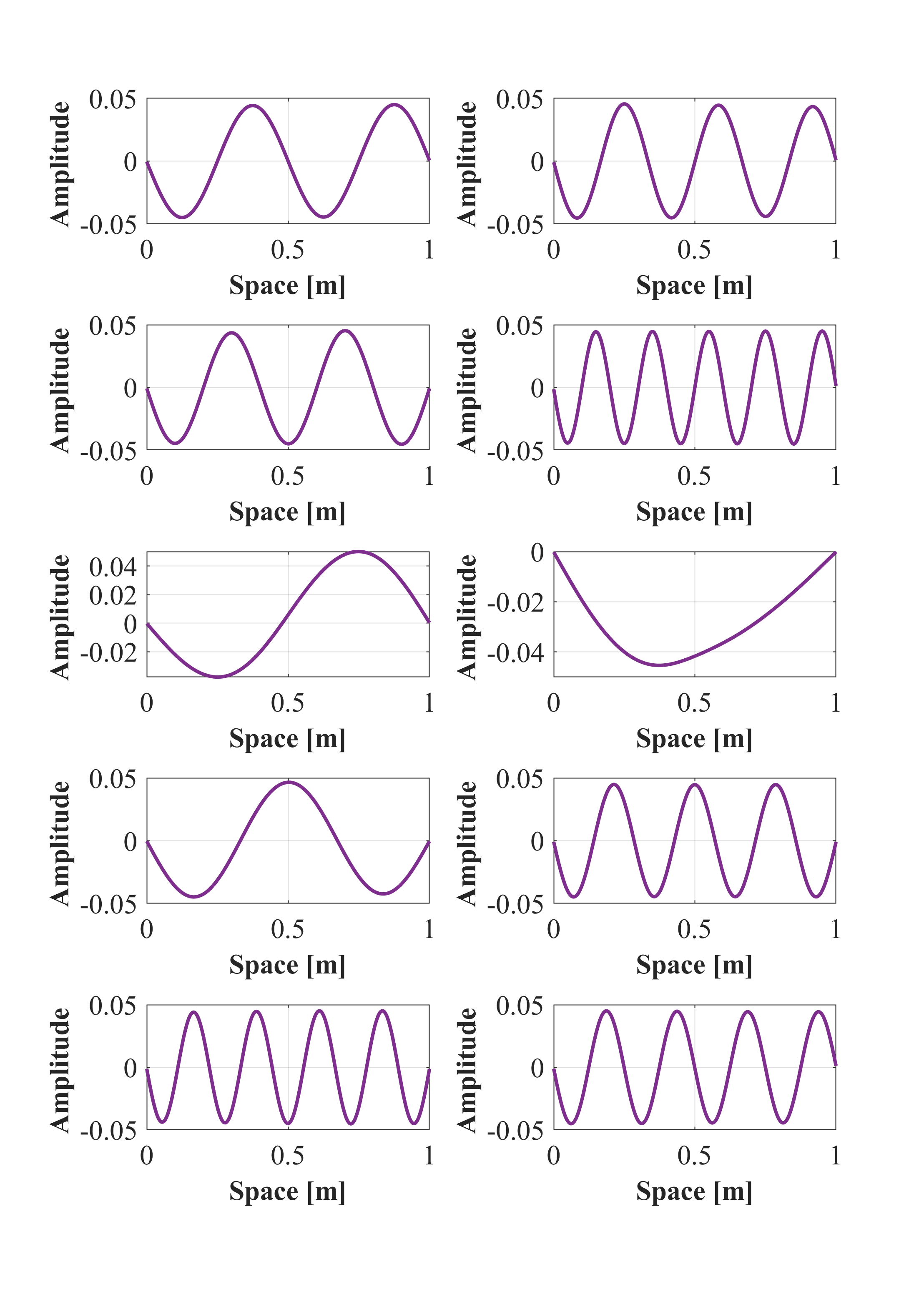}
    \caption{\label{fig:accommodations_b}Example columns of $\D$ obtained from our Wave-informed decomposition with $\gamma = 10^7$, $\lambda = 2200$.}
\end{figure}

\begin{figure}[ht]
    \centering
    \includegraphics[trim=15 10 10 15, clip, width=1.1\linewidth]{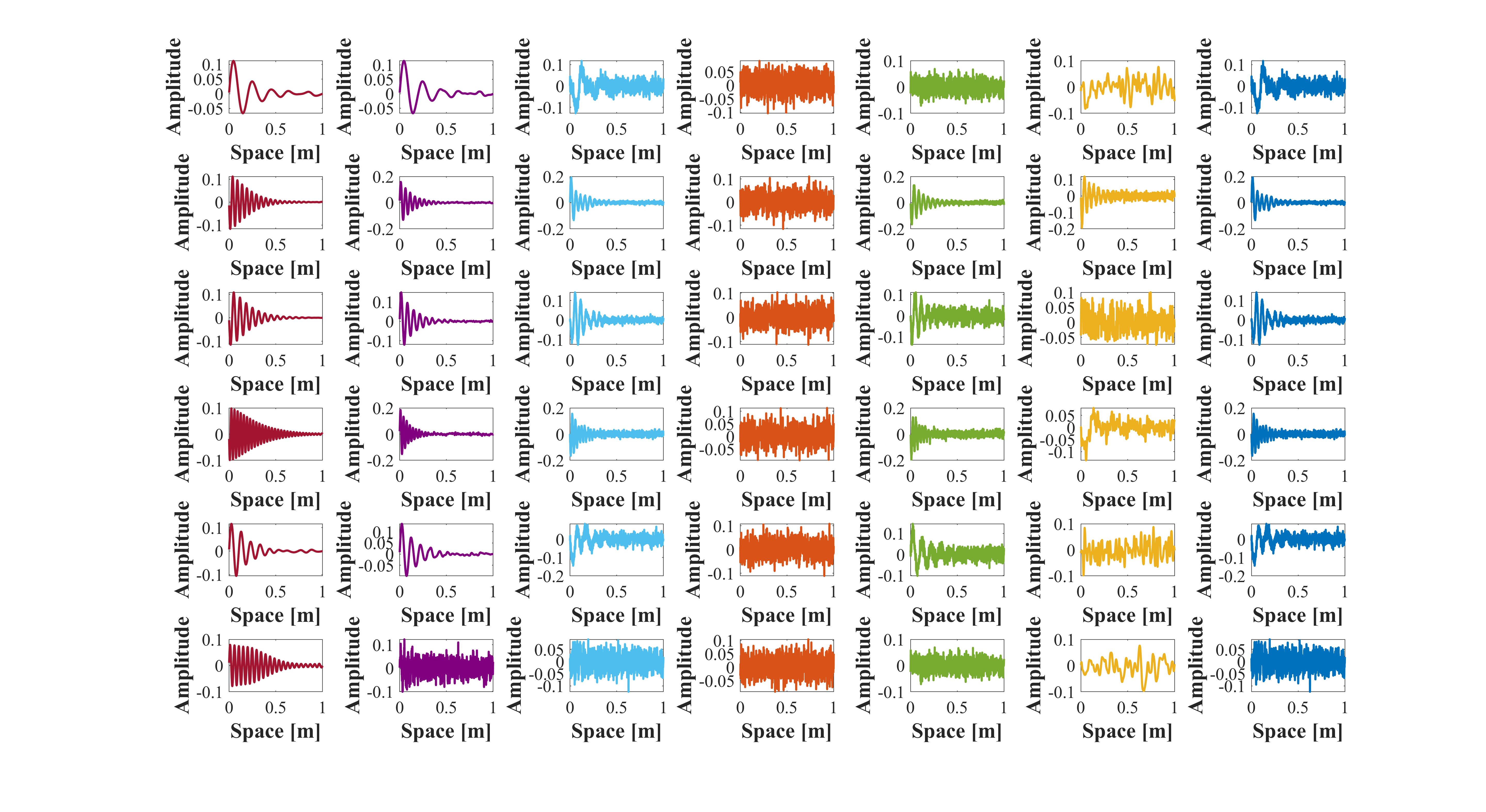}
    \caption{Recovered modes (rows) of spatially damped sinusoids (\Ai), (\B), (\C), (\Dc), (\E), (\F), (\G) when $R=6$.}    \label{fig:damped_sines_all}
    \vspace{-5mm}
\end{figure}

\end{document}